\documentclass[twoside,journal]{IEEEtran}
\ifCLASSINFOpdf
\else
\fi
\usepackage[T1]{fontenc}    
\usepackage{hyperref}       
\usepackage{url}            
\usepackage{booktabs}       
\usepackage{amsfonts}       
\usepackage{nicefrac}       
\usepackage{microtype}      
\usepackage{cite}
\usepackage{soul}
\usepackage{url}

\usepackage{xcolor}
\usepackage{graphicx}
\usepackage{subfigure}
\usepackage{epstopdf}
\usepackage{amsthm}
\usepackage{amssymb}
\usepackage{amsmath}
\usepackage{algorithm}
\usepackage{algorithmic}
\def\VectorFont{\bf}

\newcommand{\vu}{{\VectorFont u}}
\newcommand{\vv}{{\VectorFont v}}

\newif\ifsubmit

\submitfalse

\ifsubmit
\newcommand{\thui}[1]{}
\newcommand{\ruoyu}[1]{}
\newcommand{\zhiguo}[1]{}
\else
\newcommand{\thui}[1]{{\color{blue}{[Tsung-Hi: #1]}}}
\newcommand{\ruoyu}[1]{{\color{red}{[Ruoyu: #1] }}}
\newcommand{\zhiguo}[1]{{\color{blue}{[Zhiguo: #1]}}}
\fi

\newtheorem{theorem}{Theorem}
\newtheorem{lemma}{Lemma}
\newtheorem{remark}{Remark}

\newtheorem{property}{Property}
\newtheorem{assumption}{Assumption}

\begin{document}
%
\title{Federated Semi-Supervised Learning with Class Distribution Mismatch}
%
%
%
%
\author{Zhiguo~Wang,
        Xintong~Wang,
        Ruoyu~Sun and
        Tsung-Hui~Chang
\thanks{ Zhiguo Wang is with College of Mathematics, Sichuan
University, Chengdu, Sichuan 610064, China (e-mail: wangzhiguo@scu.edu.cn). The work of Zhiguo Wang is supported by the Fundamental Research Funds for the Central Universities.}
\thanks{ Xintong Wang and Tsung-Hui Chang are with the Chinese University of Hong Kong, Shenzhen
and also with Shenzhen Research Institute of Big Data. (e-mail: 220019010@link.cuhk.edu.cn;
tsunghui.chang@ieee.org). Corresponding
author: Tsung-Hui Chang. The work of T.-H. Chang is supported in part by the Shenzhen Fundamental Research Fund under Grant JCYJ20190813171003723 and in part by the NSFC, China, under Grant 61731018 and 62071409.}
\thanks{Ruoyu Sun is with Department
of Industrial and Enterprise Systems
Engineering, University of Illinois at Urbana-Champaign (e-mail: ruoyus@illinois.edu).}
}

%
%

\markboth{IEEE TRANSACTIONS ON NEURAL NETWORKS AND LEARNING SYSTEMS,~Vol.~X, No.~X, October 2021}%
{Wang \MakeLowercase{\textit{et al.}}: Federated Semi-Supervised Learning with Class Distribution Mismatch}
%



\maketitle


\begin{abstract}
Many existing federated learning (FL) algorithms are designed for supervised learning tasks, assuming that the local data owned by the clients are well labeled. However, in many practical situations, it could be difficult and expensive to acquire complete data labels. Federated semi-supervised learning (Fed-SSL) is an attractive solution for fully utilizing both labeled and unlabeled data.  Similar to that encountered in federated supervised learning, class  distribution of labeled/unlabeled data could be non-i.i.d. among clients. Besides, in each client, the class distribution of labeled data may be distinct from that of unlabeled data. Unfortunately, both can severely jeopardize the FL performance. To address such challenging issues, we introduce two proper regularization terms that can effectively alleviate the class distribution mismatch problem in Fed-SSL. In addition, to overcome the non-i.i.d. data, we leverage the variance reduction and normalized averaging techniques to develop a novel Fed-SSL algorithm. Theoretically, we prove that the proposed method has a convergence rate of $\mathcal{O}(1/\sqrt{T})$, where $T$ is the number of communication rounds, even when the data distribution  are non-i.i.d. among clients. To the best of our knowledge, it is the first formal convergence result for Fed-SSL problems. Numerical experiments based on MNIST data and CIFAR-10 data show that the proposed method can greatly improve the classification accuracy compared to baselines.
\end{abstract}




%
\IEEEpeerreviewmaketitle

\vspace{-0.3cm}
\section{Introduction}
\vspace{-0.0cm}
Federated optimization \cite{McMahan2017,li2019federated,yang2019federated}, or FL, is proposed to enable collaborative machine learning (ML) without explicit sharing the client's raw data.
Compared with the traditional distributed learning setting \cite{chang2020distributed}, FL faces four challenges: non-i.i.d. data distribution, unbalanced data distribution, a massive number of clients and limited network connection \cite{konevcny2016federated}. To address this issue, various FL algorithms have been proposed \cite{yang2019federated}. Among them, the federated averaging (FedAvg) algorithm proposed in \cite{McMahan2017} has drawn significant attention due to its simplicity and communication efficiency. Specifically, FedAvg employs the local stochastic gradient descent (SGD) \cite{stich2018local} in parallel at the local devices, followed by model averaging at the server in each communication round. Recent efforts \cite{li2019convergence,khaled2019first} have established the convergence of FedAvg when the objective function is convex.
For a nonconvex FL optimization, in  \cite{karimireddy2020scaffold}, the authors obtained a tight convergence rate for FedAvg. However, FedAvg is known to
 suffer  slow convergence when the data  among clients are non-i.i.d.. Therefore, various  techniques such as adding a proximal term
(FedProx \cite{li2018federated}),  variance reduction
(VRL-SGD \cite{liang2019variance}, SCAFFOLD \cite{karimireddy2020scaffold}), and  normalized averaging (FedNova \cite{wang2020tackling}) are proposed.

Despite the popularity, most of the existing works on FL focused on the supervised learning tasks where the data owned by the clients are well labeled.
Nevertheless, data labeling can be expensive
and time consuming. 
This issue is more severe under the FL setting, since
clients may not have the resources to provide labels for their personal data, e.g., pictures in mobile phones or medical images in hospitals \cite{cohen2004}.
It raises a fundamental question of how to effectively make use of these massive and distributed unlabeled data for improving FL.
In ML, a standard solution to utilize unlabeled data is semi-supervised learning (SSL) methods \cite{cohen2004,zhu2005,chapelle2009}.
Therefore, it is reasonable to consider  SSL for FL.
\vspace{-0.3cm}
\subsection{Related Works}
We provide a brief review of existing SSL methods for data classification.
According to \cite{zhu2005,chapelle2009}, traditional centralized SSL methods include self-training, graph-based methods and semi-supervised support vector machines (SVMs), to name a few.
The success of these methods relies on some critical assumptions, in addition to the standard smoothness, cluster and manifold assumptions.
For example,  the self-training method originated from \cite{yarowsky1995unsupervised} by Yarowsky is a well-known bootstrapping algorithm, which relies on the assumption that  the data have a well-separated clustering structure \cite{zhu2005}. Some theoretical analyses support the effectiveness of self-training algorithms in \cite{abney2004understanding,haffari2007analysis}. Many recent approaches for semi-supervised learning advocate to train a neural network based on the consistency loss, which forces the model to generate consistent outputs when its inputs are perturbed, such as the pseudo-labeling \cite{lee2013pseudo}, the ladder network \cite{rasmus2015}, the $\Pi$ model \cite{laine2016}, the mean teacher \cite{tarvainen2017}, the Virtual adversarial training (VAT) \cite{miyato2018virtual} and the Mixmatch \cite{berthelot2019mixmatch}.

Currently, the existing SSL methods often assume that labeled data and unlabeled data come from the same class distribution, but in practice the unlabeled data are unlikely to be manually purified beforehand \cite{chen2020semi}. For example, in
medical diagnosis, unlabeled medical
images may contain some rare diseases  that never appeared in the labeled data set.  As
illustrated in Fig. \ref{fig:mismatch}, for an image classification task for Client A, labeled images contain two classes (bird and dog),  but unlabeled images include three novel classes (deer, car and horse) that are not present in the labeled images.  Hence, the unlabeled data may consist of both relevant and irrelevant data. However, using the irrelevant unlabeled data often leads to  performance degradation for SSL \cite{oliver2018realistic}. We call such problem as the \emph{class distribution mismatch} problem. Recently, there are some
attempts to overcome this problem for centralized SSL. For example, reference \cite{guo2020safe} applied a  safe deep
SSL method to alleviate the harm caused by the irrelevant unlabeled data, and the key idea is to select some relevant unlabeled data rather than using all unlabeled data directly. However, such method requires solving a complicated optimization problem for data selection.

A relevant direction is to consider distributed SSL, which has received considerable interests recently \cite{chang2017distributed,avrachenkov2016distributed,scardapane2016distributed,liu2018distributed,fierimonte2016fully,xie2019distributed} . In
\cite{scardapane2016distributed,liu2018distributed}, the authors considered semi-supervised SVMs under distributed setting, and used consensus-constrained distributed optimization methods \cite{chang2014multi,zeng2018nonconvex}. In \cite{fierimonte2016fully,xie2019distributed}, manifold regularized SSL methods are studied,
which however requires the clients to estimate the Euclidean distance matrix of data samples in advance.
This may cause significant communication overhead when the system has massively distributed clients and when the data at the clients have very different distributions.
With the development of federated learning, the authors of \cite{jin2020survey}
make a brief prospect into Fed-SSL.
The recent work \cite{albaseer2020exploiting} proposed a semi-supervised FL method called FedSem to exploit the unlabeled data in smart city applications. FedSem requires two steps: the first step is to train a
global model using only the labeled data and the second step is to inject unlabeled data into the learning process using the pseudo labeling technique. In \cite{Jeong2021federated}, the author propose  a new inter-client consistency loss that regularizes the models learned at local clients to output the same prediction. Reference \cite{yang2021federated} also uses a Fed-SSL technique based on the pseudo labeling technique to aid the diagnosis of COVID-19. However, these methods in \cite{albaseer2020exploiting,Jeong2021federated,yang2021federated} neither consider the non-i.i.d. data among clients nor consider the class distribution mismatch problem between labeled data and unlabeled data.
\begin{figure}[t]
\begin{center}
\includegraphics[width=1.03\linewidth]{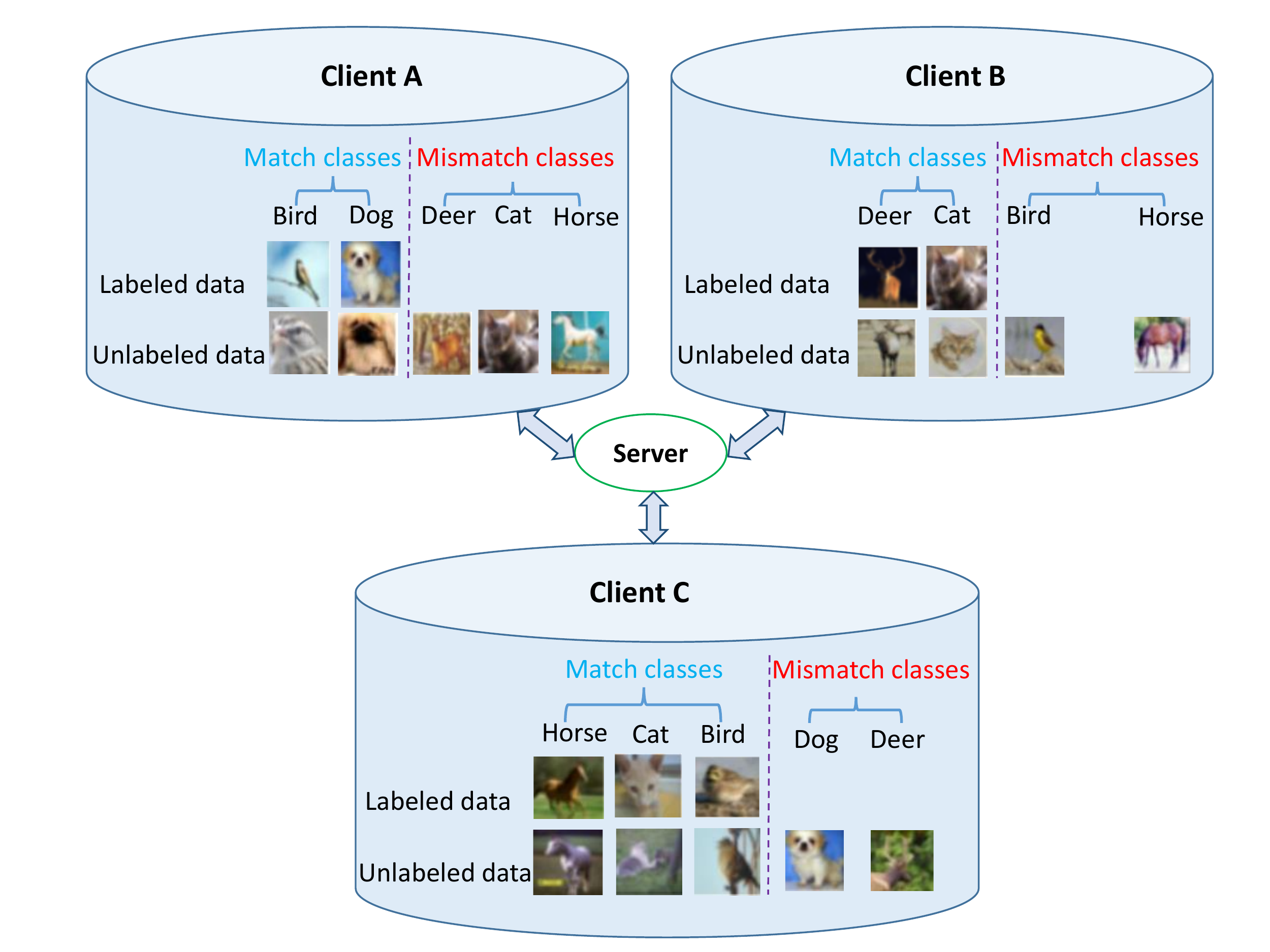}
\vspace{-0.3cm}
\caption{Fed-SSL under class distribution mismatch between labeled data unlabeled data. }
\label{fig:mismatch}
\end{center}
\vspace{-0.6cm}
\end{figure}

It is worthwhile to point out that there exist both challenges and opportunities for Fed-SSL under the two issues. The challenge lies in that non-i.i.d. data slow down the FL algorithm convergence, and thus the clients hardly make an accurate prediction for unlabeled data in early iterations.  It is conceivable that an early mistake of the pseudo labeling method can reinforce itself by generating incorrectly labeled data. Thus, re-training with these data will lead to an even worse model in successive iterations \cite{zhu2005}. On the other hand, the  opportunity lies in that the FL setting provides a means to leverage the data information from other clients (see Fig. \ref{fig:mismatch}) so that the class distribution mismatch problem can be alleviated and it helps the clients predict correct labels for the unlabeled data. This also brings a new challenge: \emph{How to transfer the knowledge of other clients to help predict labels
of data with classes that are not seen in the local labeled data subset}? These factors  make a naive combination of existing FL algorithms and centralized SSL techniques hardly
 to  deliver satisfactory performance.

\vspace{-0.3cm}
\subsection{Contributions}
In this paper, we develop a new approach to handle the Fed-SSL problem under non-i.i.d. data and class
distribution mismatch, aiming at  fully utilizing both the labeled and unlabeled data to achieve a high quality
FL performance in a communication-efficient way. Our contributions are summarized as follows.



\begin{itemize}
\item {\bf Problem formulation:} We formulate the
  Fed-SSL problem as a \emph{joint optimization problem} of  the model parameters and  the  pseudo labels of the unlabeled data. To eliminate performance degradation caused by the non-i.i.d. data and the class distribution mismatch, we introduce two regularization terms in the objective function (see \eqref{eqn: FW_joint_obj}). One is a penalty term for the pseudo labels, targeting at  boosting the prediction accuracy at early training stages. The second is a confidence penalty for the model parameters, which facilitates the knowledge transfer between the clients so the unlabeled data can be classified into some novel classes that are not seen in the labeled data.
  \item {\bf Algorithms design:} We propose a novel  \underline{fed}erated \underline{S}SL
algorithm not only with \underline{h}eterogeneous local SGD iterations but also \underline{v}ariance
\underline{r}eduction technique,  called Fed-SHVR. Since the Fed-SSL problem involves two blocks of variables, block coordinate
descent (BCD) method is known for its effectiveness in handling the loss function with multiple blocks.  To improve communication efficiency, similar to FedAvg, within each communication round of  Fed-SVHR, each client performs multiple epochs of local SGD with respect to model parameters and one step of pseudo label prediction. By recognizing the fact that heterogeneous local SGD iterations suffer from large variance and objective inconsistency when the data distribution among clients are  non-i.i.d. \cite{wang2020tackling}, inspired by \cite{karimireddy2020scaffold,wang2020tackling}, we introduce  gradient correction terms at clients to reduce the variance among clients, and normalized averaging of local gradients at the server to ensure objective consistency.

  \item {\bf Convergence analysis:}
  We prove that Fed-SHVR converges to a stationary point of the Fed-SSL problem in a sublinear rate of $\mathcal{O}(1/\sqrt{T})$, where $T$ is the number of communication rounds. The convergence analysis neither requires an assumption on the boundedness of gradient dissimilarity in \cite{wang2020tackling} nor the convexity of the objective function. To the best of our knowledge, such convergence results for Fed-SSL  have not been
presented in the literature.
  \item {\bf Experiment:} The performance of Fed-SHVR is evaluated by  experiments on the MNIST dataset and CIFAR-10 dataset.
  The experiment results show that  under non-i.i.d. data and class distribution mismatch among clients,
  Fed-SHVR can greatly improve the classification accuracy compared with baselines \cite{Jeong2021federated,zhang2020benchmarking}, including  naive combinations of FL  with centralized SSL methods.
\end{itemize}

 {\bf Synopsis:} Section \ref{sec: PF} presents the formulated Fed-SSL optimization problem. In Section \ref{sec: FedSAvg}, the proposed Fed-SVHR algorithm is presented. Section \ref{sec: CA} presents  the convergence conditions and convergence rate of  Fed-SVHR. The performance of the  Fed-SVHR algorithm is illustrated in
 Section \ref{sec: Simu}.
Finally, the conclusion is given in Section \ref{sec: conc}.

\section{Problem Formulation}\label{sec: PF}
In this section, we review the centralized SSL problem 
and present the considered novel Fed-SSL optimization problem.

\subsection{Centralized SSL via Pseudo-Labeling}
Consider a semi-supervised classification problem with a labeled dataset
$\mathcal{L}=\{(x_{i},y_{i}), i=1,\ldots,N\}$, and an unlabeled dataset $\mathcal{U}=\{u_{i}, i=1,\ldots,M\}$.
Here, $N$ and $M$ (often $N\ll M$) are the numbers of labeled and unlabeled data samples, respectively;
 $x_i$ is the $i$-th labeled sample with label $y_i\in\mathbb{R}^C$, which is a one-hot vector representing the true class label and $C$ is the number of classes.
For supervised learning, one can train a classification model by minimizing the following cross entropy loss function
\begin{align*}
L_{\rm CE}(\theta;\mathcal{L}) = -\frac{1}{N} \sum_{i=1}^{N}\Big\langle y_{i},\log\big(f_{\theta}(x_{i})\big)\Big\rangle,
\end{align*}
where $\theta$ is the classification model parameter, $f_{\theta}(x_{i})\in\mathbb{R}^C$ is the predicted probability vector (e.g., the softmax output) for each data $x_{i}$.

In order to utilize the unlabeled data $\mathcal{U}$,
we follow the pseudo-labeling method and
denote $\hat v_i\in\mathbb{R}^C$ as the pseudo label for each unlabeled sample $u_i$, for $i\in [M]\triangleq \{1,\ldots,M\}$.
In particular, we consider the following joint model training and pseudo label prediction problem \cite{Tanaka12018}
\begin{subequations}\label{problem}
\begin{eqnarray}
 \label{joint1}&&\min_{\theta,\hat{v}}
  					\ell(\theta,\hat{v})\triangleq L_{\rm CE}(\theta;\mathcal{L})
  					+\alpha_0 L_{\rm CE}(\theta;\mathcal{U},\hat{v})\\
 \label{joint2} && {\rm s.t.} ~e^\top\hat{v}_i=1,~\hat{v}_i\geq0,~ \forall~ i\in[M],
\end{eqnarray}
\end{subequations}
where $\hat{v}=\{\hat{v}_{1},\ldots,\hat{v}_{M}\}$ is the set of pseudo labels and $e=[1,\ldots,1]^\top\in \mathbb{R}^C$ is an all-one vector. As seen from \eqref{problem}, the pseudo labels $\hat v$ are treated as prediction variables and are jointly optimized with the model parameter $\theta$.
The simplex constraints in \eqref{joint2} imply that the obtained $\hat{v}$ are soft labels.
The objective function in \eqref{joint1} is composed of a supervised training loss of the labeled data and a training loss of the unlabeled data using pseudo labels. The weight $\alpha_0$ is used to balance the supervised loss and unsupervised loss \cite{lee2013pseudo}.

 To solve \eqref{problem}, one can use the popular block coordinate descent  (BCD) method by updating $\theta$ and $\hat{v}$ in an alternating fashion, as described below.
\begin{itemize}
  \item {\bf Updating $\hat{v}$ with fixed $\theta$:} It corresponds to the following problem
\begin{subequations}\label{problem2}
	\begin{eqnarray}
	\label{joint11}&&\min_{\hat{v}} L_{\rm CE}(\theta;\mathcal{U},\hat{v})\\
	\label{joint21} && {\rm s.t.} ~e^\top\hat{v}_{i}=1,~\hat{v}_{i}\geq0, ~ \forall~ i\in[M],
	\end{eqnarray}
\end{subequations}
From the definition of cross entropy, the objective function \eqref{joint11} is linear with respect to $\hat{v}$. Thus, the following closed solutions of  \eqref{problem2} can be readily obtained as
\begin{align}\label{eqn: closed_form}
[\hat{v}_{i}]_j=\left\{\begin{array}{ll}{1} & {\text { if } j=\arg\max\limits_{s\in \{1,2,\ldots,C\}} [f_{\theta}(u_{i})]_s,} \\ {0} & {\text { otherwise, }}\end{array}\right.
\end{align}
for $i\in [M]$, where $[\hat{v}_{i}]_j$ and $[f_{\theta}(u_{i})]_s$ denote the $j$-th  and the $s$-th entry of $\hat{v}_{i}$ and $f_{\theta}(u_{i})$, respectively.

%
  \item {\bf Updating $\theta$ with fixed $\hat{v}$:}  It corresponds to
      \begin{eqnarray}
       \label{joint12}\min_{\theta} \ell(\theta,\hat{v}), 
      \end{eqnarray}
      which can be handled by the standard SGD method \cite{berthelot2019mixmatch}.
\end{itemize}
Note that the alternating updates in  \eqref{eqn: closed_form} and \eqref{joint12}
are exactly the same as the pseudo-labeling method in \cite{lee2013pseudo}. Thus, the pseudo-labeling method \cite{lee2013pseudo} is in fact an application of the BCD method to the joint model training and pseudo label prediction problem \eqref{problem}.
Based on \eqref{problem}, we formulate a Fed-SSL problem  next.
\subsection{Fed-SSL Optimization Problem}
Consider an FL setting with a server and $K$ distributed clients.
We assume that both the labeled data $\mathcal{L}$ and unlabeled data $\mathcal{U}$ are distributed in the $K$ clients.
Specifically, for each client $k$, it owns local dataset $\mathcal{D}_k=\mathcal{L}_k\cup \mathcal{U}_k$, where
$
\mathcal{L}_k=\{(x_{k,i},y_{k,i}), i=1,\ldots,N_k\},
$ is the local training dataset and
$\mathcal{U}_k=\{ u_{k,i}, i=1,\ldots,M_k\}
$ is the local unlabeled dataset.
Here, $\mathcal{D}_k$, $k\in[K]\triangleq\{1,\ldots,K\}$ are non-overlapped, and
$N_k$ and $M_k$ are the numbers of labeled samples and unlabeled samples of client $k$, respectively.
Note that under the FL scenario, the data size $N_k$ and $M_k$ could be unbalanced among the clients. Moreover, as shown in Fig. \ref{fig:mismatch}, the labeled data and the large unlabeled data may  be drawn from different class distributions, e.g., the class distribution mismatch.

Let $\hat{\vv}=\{\hat{\vv}{_1},\ldots,\hat{\vv}{_K}\}$, where $\hat{\vv}{_k}=\{\hat{v}_{k,1},\ldots,\hat{v}_{k,M_k}\}$ is the set of pseudo labels for the $k$-th client with $\hat{v}_{k,i}\in\mathbb{R}^C$, and it belongs to the feasible set $\mathcal{V}_k=\{\hat{\vv}_k~|~e^\top\hat{v}_{k,i}=1,~\hat{v}_{k,i}\geq0, i\in [M_k]\}$. Then, we propose the following Fed-SSL optimization problem
\eqref{problem2} as
\begin{subequations}\label{dist prob}
	\begin{align}
	\min_{ \substack{\theta,\hat{\vv}}}~ &  F(\theta,\hat{\vv})\triangleq \sum_{k=1}^K\omega_k\Big[\underbrace{\ell_k(\theta,\hat{\vv}_k)
+\alpha_1r_1(\hat{\vv}_k)+\alpha_2r_2({\theta})}_{\triangleq F_k(\theta,\hat{\vv}_k)}\Big]\label{eqn: FW_joint_obj} \\
	{\rm s.t.}~ & \hat{\vv}_k\in\mathcal{V}_k,~ k\in [K],   \label{eqn: FW_joint_c}
	\end{align}
\end{subequations}
where
$\ell_k(\theta,\hat{\vv}_k)=L_{\rm CE}(\theta;\mathcal{L}_k)+\alpha_0 L_{\rm CE}(\theta; \mathcal{U}_k, \hat{\vv}_k)$ and $F_k(\theta,\hat{\vv}_k)$
is the local cost function of each client $k$,  $\omega_k\geq 0$ is the weight of the $k$-th client satisfying $\sum_{k=1}^K\omega_k=1$; for example,  $\omega_k = \frac{N_k+M_k}{N+M}$, $k\in [K]$ \cite{chang2017distributed}.

Notably, compared with \eqref{joint1}, we introduce in \eqref{eqn: FW_joint_obj} two additional regularization terms
$r_1(\hat{\vv}_k)$ and $r_2(\theta)$ for the pseudo label and model parameter, respectively, where $\alpha_1$ and $\alpha_2$ are the weight coefficients. The motivations of $r_1(\cdot)$ and $r_2(\cdot)$ are explained below.

\textbf{Regularization $r_1(\cdot)$}: When the data are non-i.i.d. among clients, FL algorithms perform unstably and have slow convergence \cite{karimireddy2020scaffold}. So the model mostly makes incorrect predictions
 at the early training stages. However, the pseudo labeling method takes hard pseudo-labels in \eqref{eqn: closed_form} as ``ground truth", which therefore causes overconfident mistakes. Meanwhile, the early mistakes of the pseudo labeling method could reinforce itself and result in an even worse model in successive iterations. Thus, we require a regularization of the pseudo labels to reduce the confidence level.

Toward this goal, let us  choose  $r_1(\hat{\vv}_k)$ as follows
\begin{align}\label{r_1}
r_1(\hat{\vv}_k)=\frac{1}{M_k}\sum_{i=1}^{M_k} {\rm KL}(\hat{v}_{k,i},\vu) ,
\end{align}
where $\vu=[\frac{1}{C},\ldots,\frac{1}{C}]\in\mathbb{R}^C$ is a uniform distribution, and
${\rm KL}(\cdot,\cdot)$ means the Kullback-Leibler divergence. Adding the regularization term $r_1(\hat{\vv}_k)$ in \eqref{r_1} makes the pseudo label prediction less decisive, as seen from the lemma below.
\begin{lemma}\label{lem: pseudo}
Let $r_1(\hat{\vv}_k)$ be defined in \eqref{r_1}. For fixed $\theta$ and $k\in[K]$, the closed form solution of optimization
\begin{subequations}\label{update_y}
	\begin{align}
	\min_{ \substack{\hat{\vv}_{k}}}~ &  \alpha_0L_{\rm CE}(\theta; u_{k}, \hat{\vv}_{k})+\alpha_1 r_1(\hat{\vv}_{k})\label{eqn:update_y1} \\
	{\rm s.t.}~ & \hat{\vv}_k\in\mathcal{V}_k,   \label{eqn:update_y2}
	\end{align}
\end{subequations}
is given by
\begin{align}\label{eqn: closed1}
[\hat{v}_{k,i}]_j = \frac{[f_{\theta}(u_{k,i})]_j^{\frac{\alpha_0}{\alpha_1}}}{\sum_{j=1}^C
[f_{\theta}(u_{k,i})]_j^{\frac{\alpha_0}{\alpha_1}}}, j=1,\ldots,C,
\end{align}
where $[\hat{v}_{k,i}]_j$ denotes the $j$th-entry of $\hat{v}_{k,i}$.
\end{lemma}
\begin{proof}
See Appendix \ref{proof_lemma}.
\end{proof}
Note that the closed solution \eqref{eqn: closed1} is exactly  the sharpening function  proposed in the popular Mixmatch method \cite{berthelot2019mixmatch}. When $\alpha_1\rightarrow 0$, the output of $\hat{v}_{k,i}$ in \eqref{eqn: closed1} approaches a one-hot distribution, which is degraded to the hard pseudo label in \eqref{eqn: closed_form}.  On the contrary, the soft pseudo-label in  \eqref{eqn: closed1} eliminates overconfident mistakes, which can achieve a better recovery accuracy  \cite{Tanaka12018}. As one will see in Section \ref{sec: Simu}, the regularizer $r_1(\hat{\vv}_k)$ can help to speed up the convergence of Fed-SSL at early training stages.

\textbf{Regularization $r_2(\cdot)$}: The unlabeled data may have some unseen class when the distributions of  labeled data and unlabeled data are  mismatched. Thus, we may transfer the required knowledge from other clients. To enable this, inspired by the work \cite{pereyra2017regularizing}, we introduce a regularizer for the model output $f_{\theta}(u_k)$  as follows
\begin{align}\label{eqn: r2}
r_2(\theta)={\rm KL}(f_{\theta}(u_k),\vu).
\end{align}
As we see,  ${\rm KL}(f_{\theta}(u_k),\vu)$ would make the prediction $f_{\theta}(u_k)$ away from a categorical distribution like the one-hot vector.
This will prevent the local model from making prediction solely based on its local knowledge of seen labels. Instead, through the FL process, the model parameters are able to infer some novel classes that have never appeared  in its local labeled data.

%

\section{Proposed Fed-SHVR Algorithm}\label{sec: FedSAvg}
In this section, we present the proposed Fed-SHVR algorithm for solving problem \eqref{dist prob} and it is presented in Algorithm \ref{alg: model_avg}.

Since the Fed-SSL problem   \eqref{dist prob} involves two blocks of variables $\theta$ and $\hat \vv$, many of the exiting FL algorithms cannot be applied directly \cite{wang2021clustering}. In view of that BCD steps like in \eqref{eqn: closed_form} and \eqref{joint12} can be used to handle Fed-SSL optimization \eqref{dist prob}, we propose to train a global model for \eqref{dist prob} in a novel way that combines BCD with the local SGD strategy \cite{McMahan2017} to reduce the communication cost. Since the non-i.i.d. data always happen for FL scenarios and at the same time averaging the local SGD can cause higher variance among clients, the variance reduction techniques \cite{liang2019variance} are desired. Specifically, for each communication round $t=1,2,\ldots$, our algorithm has two parts: one is the client update  and the other is the server update.

 \textbf{Client Update:} Based on the client's local data, the model parameter $\theta$ and pseudo label $\hat{\vv}_k$ are alternatively updated by the BCD framework. Firstly, the pseudo labels for optimization \eqref{dist prob} is obtained by solving  optimization \eqref{update_y}. Noticing that the soft label in \eqref{eqn: closed1}  is derived  when $r_1(\hat{\vv}_k)$ is selected as a proper regularizer in Lemma \ref{lem: pseudo}. Secondly, after updating $\hat{\vv}_k$, we select mini-batch  labeled data $\xi_k$ and unlabeled data $\zeta_k$
  uniformly at random from $\mathcal{L}_k$ and $\mathcal{U}_k$, respectively. Then the stochastic gradient $g_k(\theta_k, \hat \vv_k)$ of the local loss function is obtained
  \begin{align}
   g_k(\theta_k, \hat \vv_k):=
        \nabla_{\theta} F_k(\theta_k, \hat \vv_k ; \xi_k,\zeta_k).
  \end{align}
   If we directly take several local  SGDs to update model parameter $\theta$ in \eqref{update_theta}, it may  cause  client-drift when the data  are non-i.i.d. among clients.  To counter this drift, SCAFFOLD \cite{karimireddy2020scaffold} and VRL-SGD \cite{liang2019variance}  introduce a gradient correction term $d_k^t$ for local SGD. Specifically, the gradient correction term $d_k^t$ in \eqref{gradient_corr} is used in \eqref{update_theta}. As we show in Remark \ref{remark_2},
   \eqref{update_theta} together with \eqref{gradient_corr} is equivalent to taking gradient descent along an estimated global gradient direction of
   \eqref{eqn: FW_joint_obj}, and thus effectively reduces the variance among clients. Finally, the  client models after $\tau_k$ local SGD updates (Steps \eqref{eqn: fedSavg theta1} -\eqref{update_theta}) are sent to the server.

\textbf{ Server Update:} If the mini-batch size for labeled data and unlabeled data are $B_l$ and $B_u$, respectively,  and  each client $k$ performs $E$ epochs, the number of local SGD iterations is $\tau_k=\max\{m_kE/B_u,N_kE/B_l\}$,  Thus, the clients are \emph{heterogeneous} if they have different $\tau_k$
(otherwise they are \emph{homogeneous}). In \cite{wang2020tackling}, the authors have proved that the standard averaging of client models \cite{McMahan2017}
$$
\theta^t =  \sum_{k=1}^K\omega_k\theta_k^{t-1,\tau_k}
$$
after heterogeneous local updates $\tau_k$ will prevent the algorithm from converging to a stationary point, and  such phenomenon is called \emph{objective inconsistency}. To deal with heterogeneous local updates, in our algorithm, the server obtains the  model $\theta^t$ by weighted averaging between the normalized local gradients  and the previous model parameter $\theta^{t-1}$ as follows
$$\theta^t = \theta^{t-1} - \eta\bar{\tau} \left( \sum_{k=1}^K\omega_k\frac{\theta^{t-1}-\theta_k^{t-1,\tau_k}}{\eta \tau_k}\right),
$$
where $\bar{\tau}=\sum_{k=1}^K\omega_k\tau_k$. Then, the server broadcasts $\theta^t$ to the clients.



\begin{remark}\label{remark_2}{\rm (Variance reduction of Fed-SVHR)
Let us see how
the update of gradient correction term $d_k$ in \eqref{gradient_corr} help to reduce the client drift variance.
Using the fact $d_k^0=0$ and  summing up all $d_k^s$ from $0$ to $t$, we have
\begin{align*}
\nonumber d_k^t &= \sum_{s=1}^t\Big(\frac{\theta^{s-1}-\theta^s}{\eta \bar{\tau}}-\frac{1}{\eta \tau_k}(\theta^{s-1}-\theta_k^{s-1,\tau_k})\Big)\\
 &=\sum_{s=1}^t\Big(\sum_{l=1}^K\omega_l\frac{\theta^{s-1}-\theta_l^{s-1,\tau_l}}{\eta \tau_l}-\frac{1}{\eta \tau_k}(\theta^{s-1}-\theta_k^{s-1,\tau_k})\Big),
\end{align*}
where the second equality is due to \eqref{eqn: fedSavg theta1}. By  the above equality over $k=1,\ldots,K$, and using the fact $\sum_{k=1}^K\omega_k=1$, we obtain
\begin{align}\label{eqn: d_k2}
\sum_{k=1}^K\omega_kd_k^t =  0.
\end{align}
In addition, using \eqref{eqn: fedSavg theta1} again, we can rewrite $d_k^t$ as below
\begin{align}
\nonumber d_k^t &= d_k^{t-1}+\sum_{l=1}^K\omega_l\frac{\theta^{t-1}-\theta_l^{t-1,\tau_l}}{\eta \tau_l}-\frac{1}{\eta \tau_k}(\theta^{t-1}-\theta_k^{t-1,\tau_k})
\end{align}
Substituting  \eqref{update_theta} into the above equality gives rise to
\begin{align}\label{eqn: d_kt}
\nonumber d_k^t&=d_k^{t-1}+\sum_{l=1}^K\frac{\omega_l}{\tau_l}\sum_{q=1}^{\tau_l}
(g_l(\theta_l^{t-1,q-1}, \hat \vv_l^{t})+d_l^{t-1})\\
\nonumber&~~~~-\frac{1}{\tau_k}\sum_{q=1}^{\tau_k}
(g_k(\theta_k^{t-1,q-1}, \hat \vv_k^{t})+d_k^{t-1})\\
& = \sum_{l=1}^K\omega_l\sum_{q=1}^{\tau_l}\frac{1}{\tau_l}
g_l(\theta_l^{t-1,q-1}, \hat \vv_l^{t})-\sum_{q=1}^{\tau_k}\frac{1}{\tau_k}
g_k(\theta_k^{t-1,q-1}, \hat \vv_k^{t})
\end{align}
where  the final equality is due to \eqref{eqn: d_k2}. With \eqref{eqn: d_kt}  substituted into \eqref{update_theta}, we have
\begin{align}\label{theta_the}
\nonumber\theta_k^{t,q}&=\theta_k^{t,q-1}-\eta  \Big(g_k(\theta_k^{t,q-1}, \hat \vv_k^{t+1})-\sum_{q=1}^{\tau_k}\frac{1}{\tau_k}
g_k(\theta_k^{t-1,q-1}, \hat \vv_k^{t})\\
&\qquad\qquad+
\sum_{k=1}^K\omega_k\sum_{q=1}^{\tau_k}\frac{1}{\tau_k}
g_k(\theta_k^{t-1,q-1}, \hat \vv_k^{t})\Big).
\end{align}
Interestingly, the above \eqref{theta_the} resembles  SAGA \cite{defazio2014saga}, where the model is updated along an estimated global gradient direction considering the data of all clients. Thus,  Fed-SHVR can be seen as an extension of variance
reduction techniques of \cite{liang2019variance} and \cite{defazio2014saga}.}
\end{remark}
\begin{algorithm}[t!]
	\caption{\texttt{Proposed Fed-SHVR}} 
	\label{alg: model_avg}
	\begin{algorithmic}[1]
		\STATE {\bfseries Input:} initial model parameters $\theta^0=\theta_1^{0, \tau_1}=\cdots=\theta_K^{0, \tau_K}$ at the server side; initial pseudo labels of $\hat{\vv}^0_{1},\cdots,\hat{\vv}_{K}^{0}$ and gradient correction term $d_1^0=\cdots= d_K^0=0$ at the clients; initial the learning rate $\eta$.
        \STATE Calculate  the local iterations $\tau_k$ completed by client $k$ and  $\bar{\tau}=\sum_{k=1}^K\omega_k\tau_k$ at server.
        \FOR{communication round $t=1$ {\bfseries to} $T$}		
		\STATE {\bfseries \underline{Server side:}} Compute
		\begin{align}
	\label{eqn: fedSavg theta1} \theta^t &= \theta^{t-1} - \eta\bar{\tau} \left( \sum_{k=1}^K\omega_k\frac{\theta^{t-1}-\theta_k^{t-1,\tau_k}}{\eta \tau_k}\right)
		\end{align}
		and broadcast $\theta^t$ to all clients.		
		\STATE {\bfseries \underline{Client side:}}~~
        \FOR{client $k = 1$ {\bfseries to} $K$ (in parallel)}
        \STATE Obtain $\hat{\vv}_{k}^{t+1}$ from \eqref{update_y}-\eqref{eqn: closed1}
  		\STATE Update gradient correction term
        \begin{align}\label{gradient_corr}
        d_k^t = d_k^{t-1}+\frac{\theta^{t-1}-\theta^t}{\eta \bar{\tau}}-\frac{\theta^{t-1}-\theta_k^{t-1,\tau_k}}{\eta \tau_k}.
        \end{align}
        \STATE Set $\theta_k^{t,0} = \theta^t$.
        \FOR{ $q = 1$ {\bfseries to} $\tau_k$}
        \STATE Select data $\xi_k^{t,q}$ and $\zeta_k^{t,q}$ uniformly at random from $\mathcal{L}_k$ and $\mathcal{U}_k$, and update
        \begin{align}\label{update_theta}
        \theta_k^{t,q}=\theta_k^{t,q-1}-\eta  \big(g_k(\theta_k^{t,q-1}, \hat \vv_k^{t+1})+d_k^t\big),
		\end{align}	
		\ENDFOR
		\STATE
		Upload $\theta_k^{t,\tau_k}$ to the server.
		\ENDFOR
		\ENDFOR
	\end{algorithmic}
\end{algorithm}
\vspace{-0.2cm}
%
%
%


\section{Convergence Analysis}\label{sec: CA}
In this section, we build the convergence conditions of the proposed Fed-SHVR algorithm.
\subsection{Assumptions}
We first make some standard assumptions.
\begin{assumption}\label{ass: reg}
The regularization terms $r_1(\hat{\vv})$ and $r_2(\theta)$ are continuous differentiable functions. In addition, $r_1(\hat{\vv})$ is a $\mu$-strongly convex, where $\mu>0$, i.e., for any $\hat \vv_1,\hat \vv_2 \in \mathcal{V}_k$
\begin{align}\label{reg_strong}
r_1(\hat{\vv}_1)-r_1(\hat{\vv}_2)-\langle\nabla r_1(\hat{\vv}_2),\hat{\vv}_1-\hat{\vv}_2\rangle \geq \frac{\mu}{2}\|\hat{\vv}_1-\hat{\vv}_2\|^2.
\end{align}
\end{assumption}
 The regularization term $r_1(\hat{\vv}_{k}) $ defined in \eqref{r_1} is strongly convex over the probabilistic simplex $\mathcal{V}_k$ with respect to the $\ell_1$-norm (see \cite{shalev2007logarithmic}, Definition 2 and Example 2). However,  note that the objective function $F_k(\theta,\hat{\vv}_k)$ in \eqref{eqn: FW_joint_obj} is not jointly convex with respect
to  $(\theta,\hat{\vv})$.

\begin{assumption}\label{ass: smooth}
The local cost $F_k(\theta,\hat{\vv}_{k})$ is $L$-smooth (possibly non-convex) with respect to $(\theta,\hat \vv_k)$ for $k\in [K]$, i.e.,
\begin{align}\label{L_smooth}
\nonumber&\|\nabla_{\theta}F_k(\theta,\hat{\vv}_{k})-\nabla_{\theta}F_k(\theta',\hat{\vv}'_{k})\|
\\
&\leq L\sqrt{\|\theta-\theta'\|^2+\|\hat{\vv}_{k}-\hat{\vv}_{k}'\|^2},
\end{align}for all $\theta, \theta'$ and $\hat \vv_k,\hat \vv'_k \in \mathcal{V}_k$.
\end{assumption}
\begin{assumption}\label{ass: un_var} Given $\hat{\vv}_k$,  $k\in [K]$, assume that the stochastic gradient satisfies the following conditions
\begin{align}
\label{eqn: unba}
&\mathbb{E}[g_k(\theta_k,\hat{\vv}_{k})]=\nabla_{\theta} F_k(\theta_k,\hat{\vv}_{k})\\
\label{eqn: var}&\mathbb{E}[\|g_k(\theta_k,\hat{\vv}_{k})-\nabla_{\theta} F_k(\theta_k,\hat{\vv}_{k})\|^2]\leq \sigma^2,
\end{align}
where $\sigma$ is the noise variance and $\mathbb{E}$  denotes the expectation with respect to all random variables $\{\xi_k,\zeta_k\}$.
\end{assumption}
Assumption \ref{ass: smooth} makes a standard smoothness assumption in non-convex optimization with two variables \cite{xu2015block}.
Assumption \ref{ass: un_var} is
a common assumption that the stochastic gradient noise is zero mean with bounded variance $\sigma^2$ \cite{karimireddy2020scaffold}.

\subsection{Convergence Analysis of Fed-SVHR}


Since variable $\hat{\vv}_{k}$ has the constraint $\mathcal{V}_k$, let us define the following optimality gap:
\begin{align}\label{eqn: gap}
g_t\triangleq \mathbb{E}\left[\|\nabla_{\theta}F(\theta^t,\hat{\vv}^{t+1})\|^2\right]+\sum_{k=1}^K\omega_k\|\hat{\vv}_{k}^{t+1}-\hat{\vv}_{k}^{t}\|^2.
\end{align}
From the definition, obviously, $g_t\geq 0$, and it has the following property.
\begin{property}\label{pro_1}
When $g_t=0$, the iterate $\{\theta^t,\hat{\vv}^t\}$ will be a stationary
point of problem \eqref{dist prob}.
\end{property}
\begin{proof}
See Appendix \ref{proof_property}.
\end{proof}
%
\begin{theorem}\label{theo: 1}
Assume Assumptions \ref{ass: reg}-\ref{ass: un_var} hold. If $\eta=\sqrt{\frac{K}{T\bar{\tau}}}\leq \{\frac{4\mu\alpha_1}{\bar{\tau}(2L^2+1 )},\frac{1}{L\sqrt{15a_{\tau}}},\frac{1}{2L\bar{\tau}}\}$, then the sequence
$\{\theta^t,\hat{\vv}^t\}$ generated by Algorithm \ref{alg: model_avg} satisfies
\begin{align}\label{eqn: outputs1}
\min_{t\in\{1,\ldots,T\}}g_t&\leq \frac{8\mathbb{E}[F(\theta^{1},\hat{\vv}^1)]}{\sqrt{KT\bar{\tau}}}
+\sqrt{\frac{K}{T\bar{\tau}}}c_1\sigma^2,
\end{align}
where  $c_1=\frac{8}{ \bar{\tau}}\Big(\frac{33}{8}a_{\tau}
+\frac{3}{2}\bar{\tau}^2\sum_{k=1}^K\frac{\omega_k^2}{\tau_k}\Big)L$, $\bar{\tau}=\sum_{k=1}^K\omega_k\tau_k$, $a_{\tau}=(\hat{\tau}-1)(2\hat{\tau}-1)$ and $\hat{\tau}=\max\{\tau_1,\ldots\tau_K\}$.
\end{theorem}
\begin{proof}
See Appendix \ref{proof_theorem}.
\end{proof}
Theorem \ref{theo: 1} shows that the proposed Fed-SHVR has a convergence rate of $\mathcal{O}(1/\sqrt{T})$. To the best of our knowledge, this is the first result that shows the convergence rate for Fed-SSL optimization.
\begin{remark}{\rm
In \cite{liang2019variance,wang2020tackling}, the authors have proved that VRL-SGD and FedNova have a  convergence rate of $\mathcal{O}(1/\sqrt{T})$.
However, their proofs cannot be applied to the proposed Fed-SVHR since Fed-SVHR combines both the variance reduction and normalized averaging techniques, in addition to that Fed-SHVR involves two blocks of variables. Besides, different from \cite{wang2020tackling}, our analysis does not require the assumption on the boundedness of gradient dissimilarity, i.e., $\sum_{k=1}^K\omega_k\|\nabla_{\theta} F_k(\theta_k,\hat{\vv}_{k})-\nabla_{\theta} F(\theta,\hat{\vv})\|^2\leq \kappa^2$, where $\kappa$ is a constant. Thus, new  proof techniques are developed in our proof in order
 to establish the convergence rate in \eqref{eqn: outputs1}; details are given in Appendix \ref{proof_theorem}.}
\end{remark}
\section{Numerical Results}\label{sec: Simu}
In this section, we examine the numerical performance of the proposed Fed-SVHR algorithm
and present comparison results with the existing methods for MNIST data and CIFAR-10 data.
\subsection{MNIST data}
Let us consider the MNIST digit recognition task.  The data is split into 60000
 images for training and 10000 images for testing.  There are $K=10$ clients to study federated optimization. In addition,
 two ways of dividing  the MNIST data over clients:
 \begin{itemize}
\item \textbf{IID}: Each client is randomly assigned a uniform distribution over 10 classes and receives 6000 examples, which contains 60 labeled data and 5940 unlabeled data.
\item \textbf{Non-IID}: Each client has only 60 labeled data of two digits, and
59400 unlabeled data are randomly partitioned
across 10 clients using a Dirichlet distribution $\textmd{Dir}_{10}(0.1)$ \cite{wang2021demystifying}. Thus  the class distribution between labeled data and unlabeled data in each client is  mismatched.
\end{itemize}
 We run each experiment with 5 random seeds
and report the average.
Consider a simple Multi-layer perceptron (MLP) with one-hidden layers including 5000 units each using ReLU activations. The number of mini-batch is $B_l=32$ for labeled data, $B_u=32$ for unlabeled data. In \eqref{update_theta}, we set the learning rate $\eta=0.01$, $\alpha_1=0.75$, $\alpha_2=0.1$ and $\alpha_{0}$ ramp up its weight from 0 to
its final value during the first 50 epochs. All
clients perform $E=2$ local epochs, then the number of local SGD iterations is $\tau_k\in [270,500]$ and $\tau_k=371$ for Non-IID case and IID case, respectively.

\textbf{Example 1:} Intuitively, the popular methods mean-teacher \cite{tarvainen2017} and pseudo-labeling \cite{lee2013pseudo} can be combined with FedAvg \cite{McMahan2017}, called Fed-MT and Fed-Pseudo, respectively. Fig. \ref{test_accuracy} presents the test accuracy of the proposed Fed-SHVR and the other algorithms including Fed-MT, Fed-Pseudo and vanilla FedAvg under IID and Non-IID cases.
It reveals that the performance of the proposed Fed-SHVR is much better than that of  Fed-MT and Fed-Pseudo whether for the IID case or  Non-IID case. Meanwhile, all of the Fed-SSL algorithms perform better than the supervised FedAvg that only uses 600 labeled data. Compared Fig. \ref{test_accuracy} (a) with Fig. \ref{test_accuracy} (b), the Non-IID case with mismatched class distribution degrades the performance of Fed-SSL. Fortunately, our method introduces two regularized terms, which can help the Fed-SSL to obtain higher test accuracy.

\textbf{Example 2:} We evaluated the performance of each federated  SSL technique on IID case with 600 labeled data samples and varying amounts of unlabeled data, which resulted in the test accuracy shown in Fig. \ref{unlabel_client_accuracy} (a). Compared with  FedAvg, increasing the number of unlabeled data tends to improve the performance of Fed-SSL
techniques while the proposed Fed-SVHR performs best. Fig. \ref{unlabel_client_accuracy} (b) shows the test accuracy curve when increasing the number of clients. From  Fig. \ref{unlabel_client_accuracy} (b), we see that the performance of Fed-SLL methods is with a slight change, thus Fed-SSL methods are robust to the number of the client while they perform better than FedAvg that only uses labeled data.
\begin{figure*}[t]
\centering
\vspace{-1cm}
\subfigure[IID]{
\begin{minipage}{0.5\linewidth}
\centering
\includegraphics[width=3in]{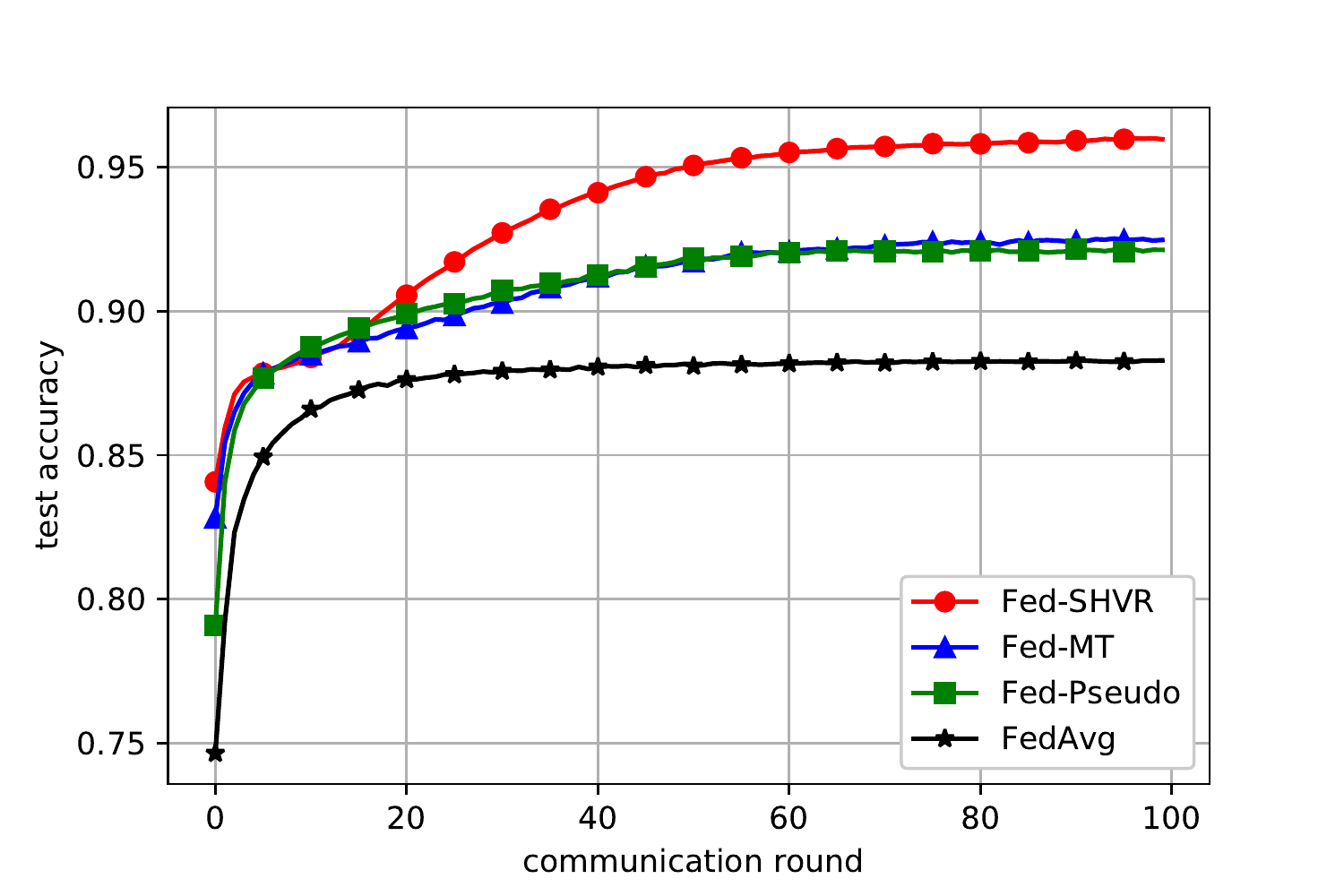}
\end{minipage}%
}%
\hspace{-2.1cm}
\subfigure[Non-IID]{
\begin{minipage}{0.5\linewidth}
\centering
\includegraphics[width=3in]{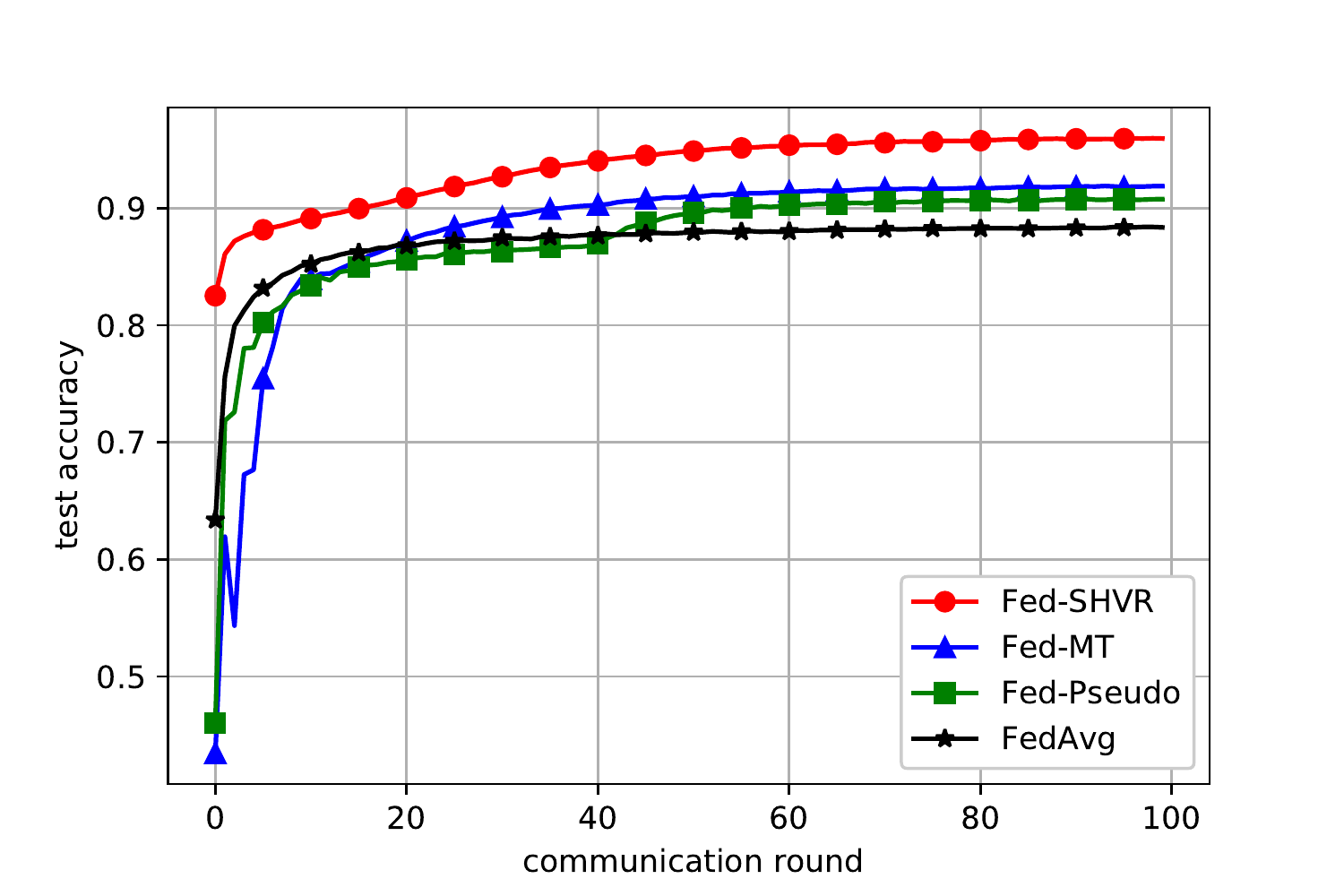}
\end{minipage}%
}%
\centering
\caption{Test accuracy curves of the proposed Fed-SVHR and the baseline on IID \& Non-IID cases when varying the number of communication rounds. }
\label{test_accuracy}
\end{figure*}
\begin{figure*}[t]
\centering
\vspace{-0.5cm}
\subfigure[unlabeled data increasing]{
\begin{minipage}[t]{0.5\linewidth}
\centering
\includegraphics[width=3in]{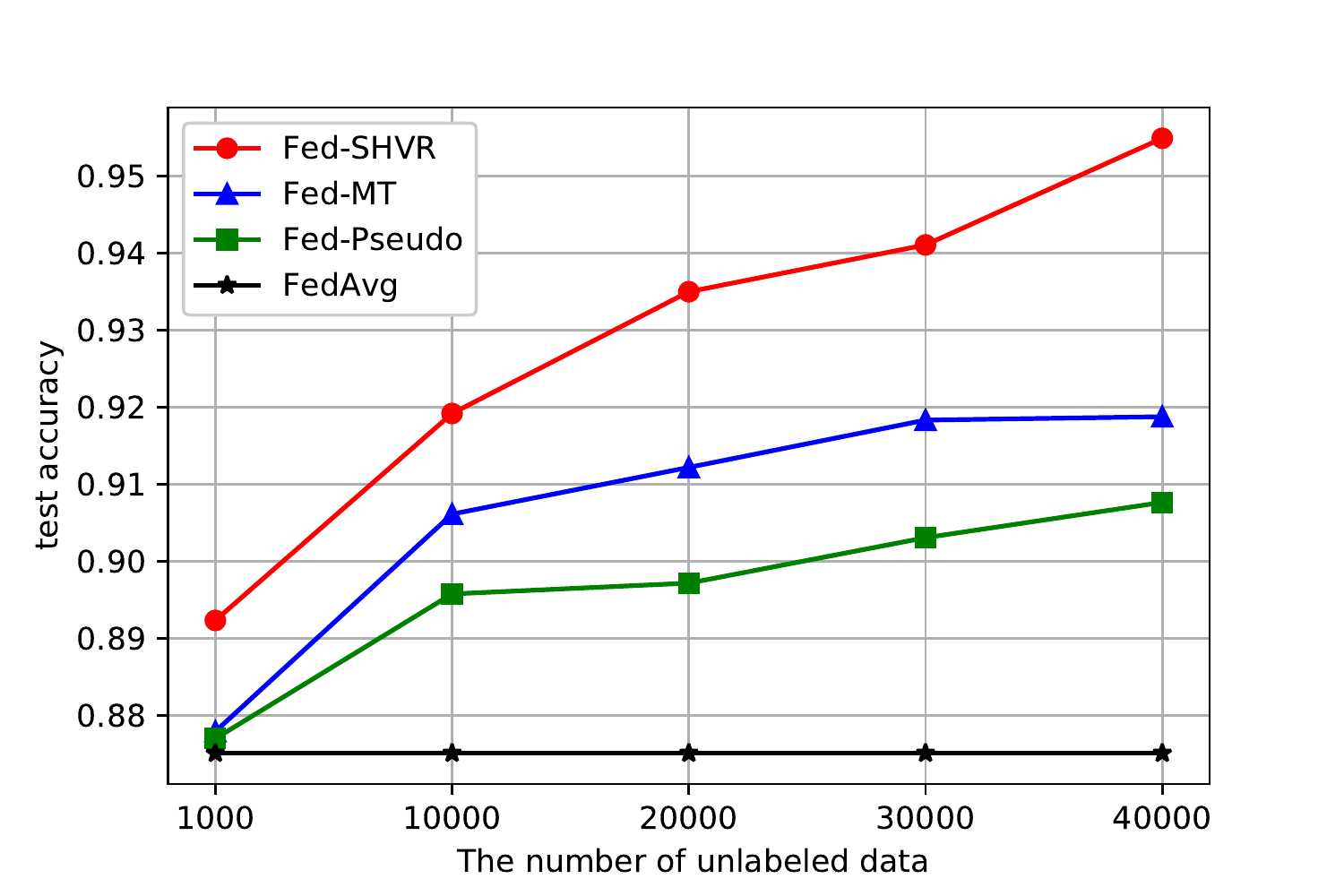}
\end{minipage}%
}%
\hspace{-2.1cm}
\subfigure[Clients increasing]{
\begin{minipage}[t]{0.5\linewidth}
\centering
\includegraphics[width=3in]{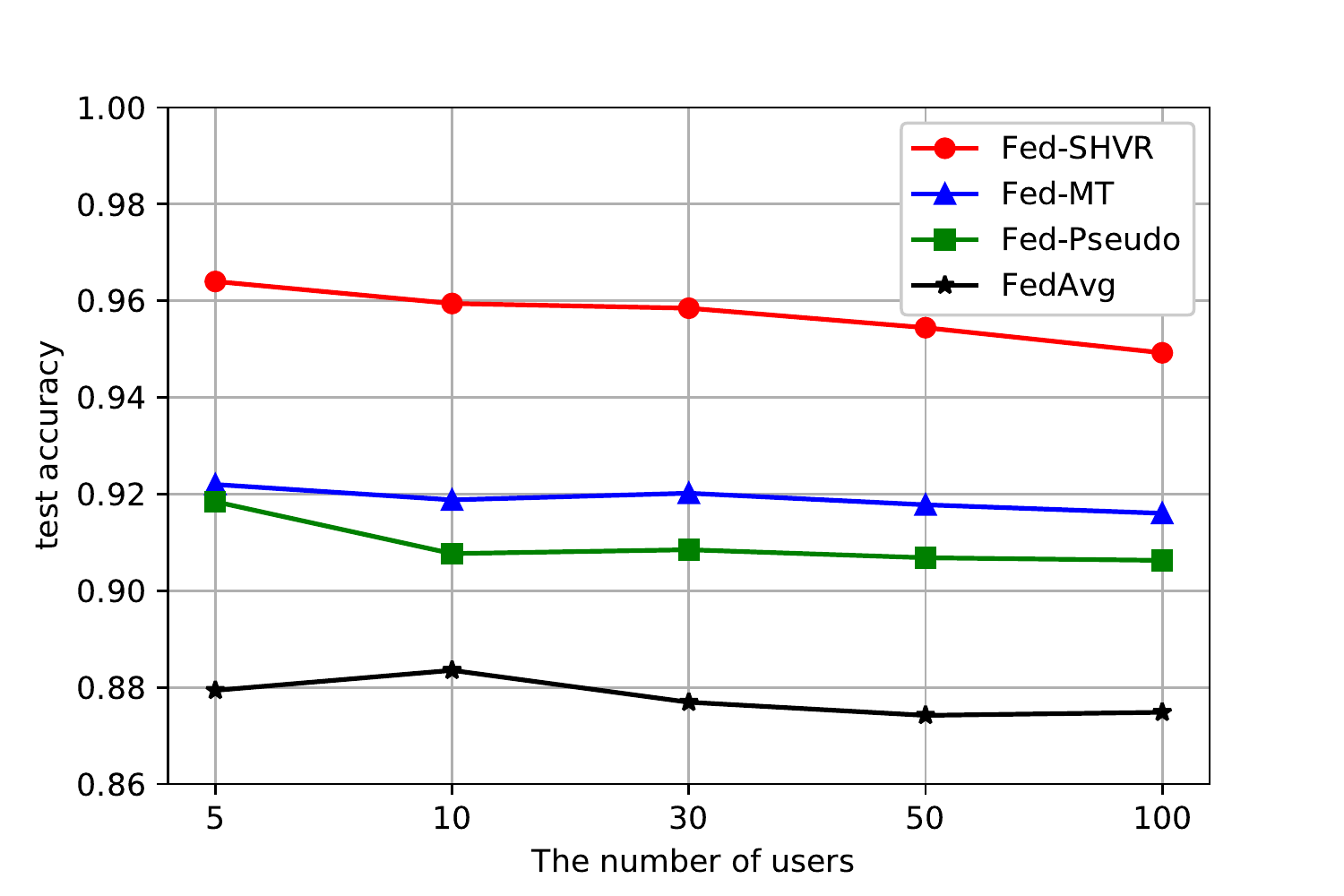}
\end{minipage}%
}%
\centering
\caption{The test accuracy curves of the proposed Fed-SVHR and the baseline on IID case when varying the number of unlabeled data and clients.
 }
\label{unlabel_client_accuracy}
\vspace{-0.2cm}
\end{figure*}

\begin{figure*}[t]
\centering
\subfigure[training loss]{
\begin{minipage}[t]{0.5\linewidth}
\centering
\includegraphics[width=3in]{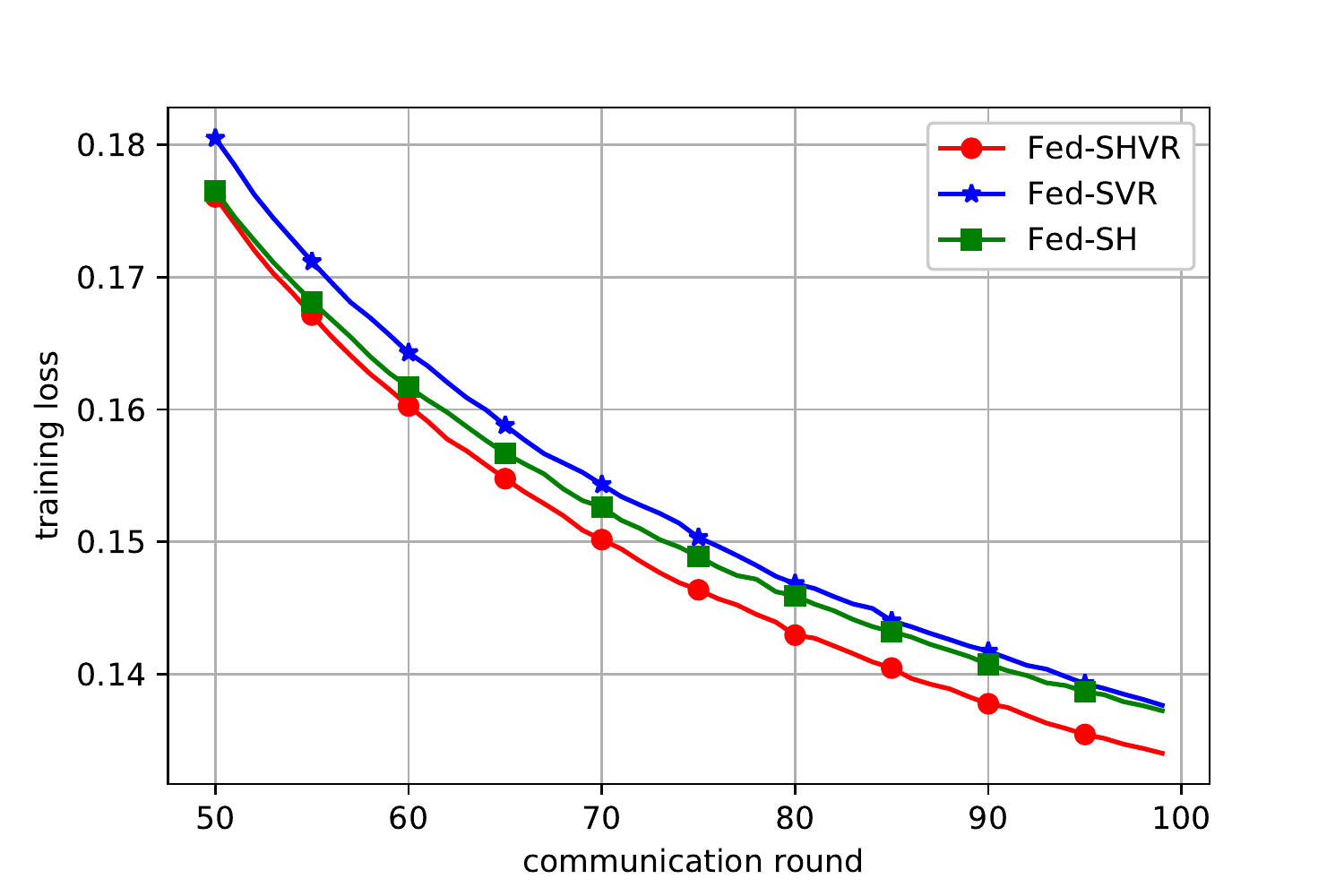}
\end{minipage}%
}%
\hspace{-2.1cm}
\subfigure[std]{
\begin{minipage}[t]{0.5\linewidth}
\centering
\includegraphics[width=3in]{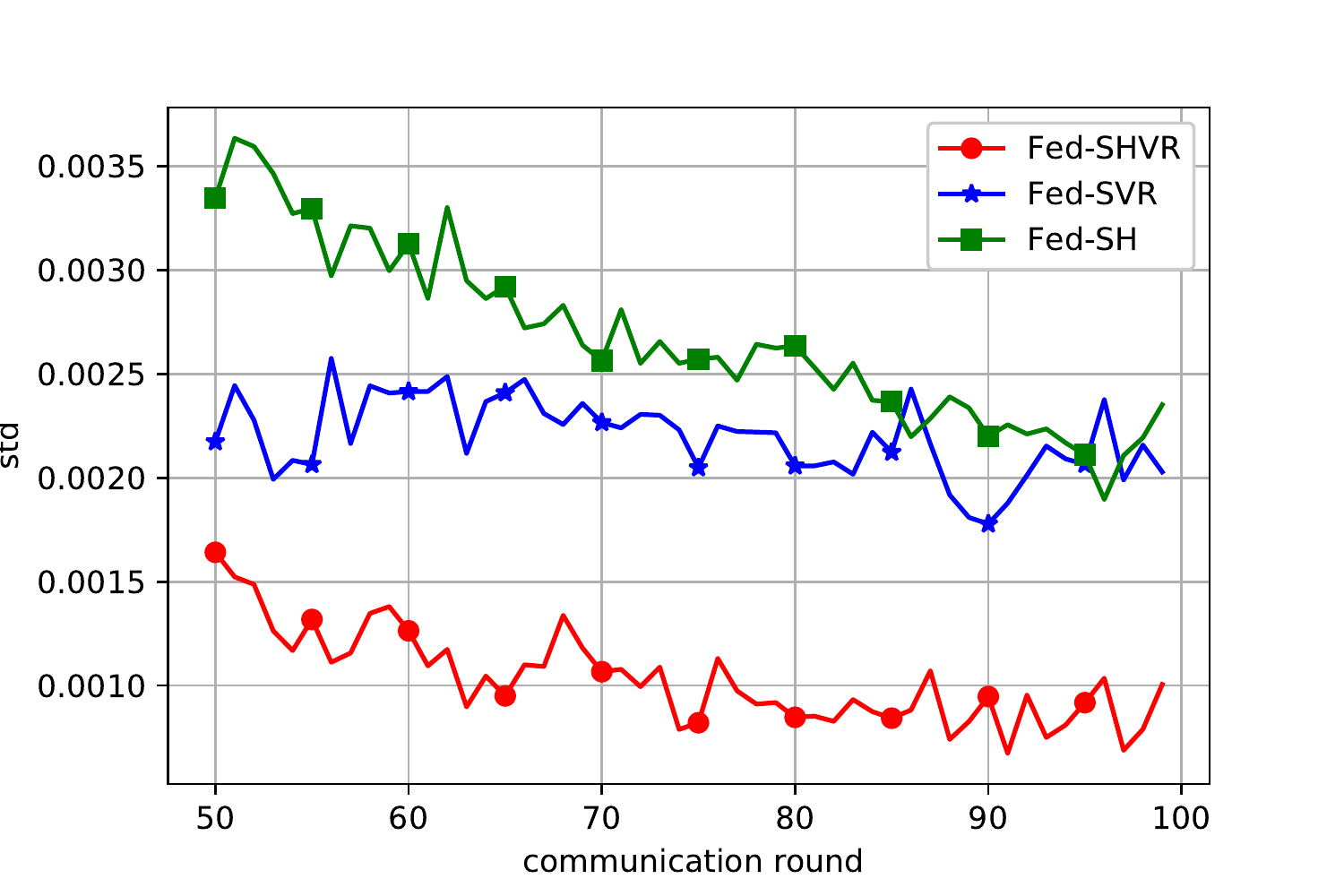}
\end{minipage}%
}%
\vspace{-0.2cm}
\centering
\caption{Comparison of proposed algorithms Fed-SVHR, Fed-SH, Fed-SVR in terms of training loss and standard deviation on Non-IID case.}
\label{training_loss}
\end{figure*}

\textbf{Example 3:} We consider the ablation study. The proposed Fed-SHVR is a combination of FedNova and VRL-SGD, which  takes heterogeneity local SGD iterations and variance reduction technique. In this subsection, if we use FedNova update model parameter without gradient correction term $d_k^t$ in \eqref{update_theta}, then Fed-SHVR degrades to Fed-SH. Similarly, we call Fed-SVR that only uses the variance reduction technique of VRL-SGD but makes heterogeneous local SGD iterations. When varying the communication round at the Non-IID setting, Fig. \ref{training_loss} shows the training loss and standard deviation (std) curve. Here, at communication round $t$,
$
 {\rm std}_t=\sqrt{\frac{\sum_{i=1}^5(p_t^i-\bar{p}_t)^2}{5}}
$
 where $p_t^i$ is the prediction accuracy of the global model for $i$-th experiment, and $\bar{p}_t=\frac{1}{5}\sum_{i=1}^5p_t^i$. From Fig. \ref{training_loss} (a), it shows that Fed-SHVR reaches the smallest training loss. Fig. \ref{training_loss} (b) reveals that the Fed-SSL with variance reduction technique can reduce the variance among workers while Fed-SHVR has the smallest variance.

\begin{table}[t]
\centering
\caption{Test accuracy of the proposed Fed-SHVR by using different parameters $\alpha_1$ and $\alpha_2$ on IID case \& Non-IID case. }
\renewcommand\tabcolsep{2.5pt}
\begin{tabular}{|l|l|l|l|l|l|}
\hline \multicolumn{2}{|c|}{} & \multicolumn{2}{|c|}{IID} & \multicolumn{2}{|c|}{Non-IID} \\
\hline \multicolumn{2}{|c|}{Communication round} & $T=1$&$T=100$ &$T=1$&$T=100$ \\
\hline \multicolumn{2}{|c|}{FedAvg with 600 labeled data}  &58.5\%&88\% & 37.8\% &87.5\%\\
\hline \multicolumn{2}{|c|}{FedAvg with 60000 labeled data}  &94.4\%&98.5\% & 94.3\% &98.4\%\\
\hline  & $\alpha_1=\alpha_2=0$&81.4\% &92.5\% &34.7\% &92.4\% \\
\cline { 2 - 6 } Fed  & $\alpha_1=0.75,\alpha_2=0$ &83.5\%&92.7\%&\textbf{74.4 \%}&\textbf{92.7\%} \\
\cline { 2 - 6 } -SHVR& $\alpha_1=0,\alpha_2=0.1$ &83.1\%&95.2\%&\textbf{72.9\%} & \textbf{95.4\%}\\
\cline { 2 - 6 } & $\alpha_1=0.75,\alpha_2=0.1$ & 83.6\%&96.6\%&82.2\%&95.5\% \\
\hline
\end{tabular}
\label{tab: accuracy_para}
\end{table}
\begin{figure*}[t]
\centering
\subfigure[Fed-SHVR with $\alpha_1=0.75, \alpha_2=
0$]{
\begin{minipage}[t]{0.5\linewidth}
\centering
\includegraphics[width=3in]{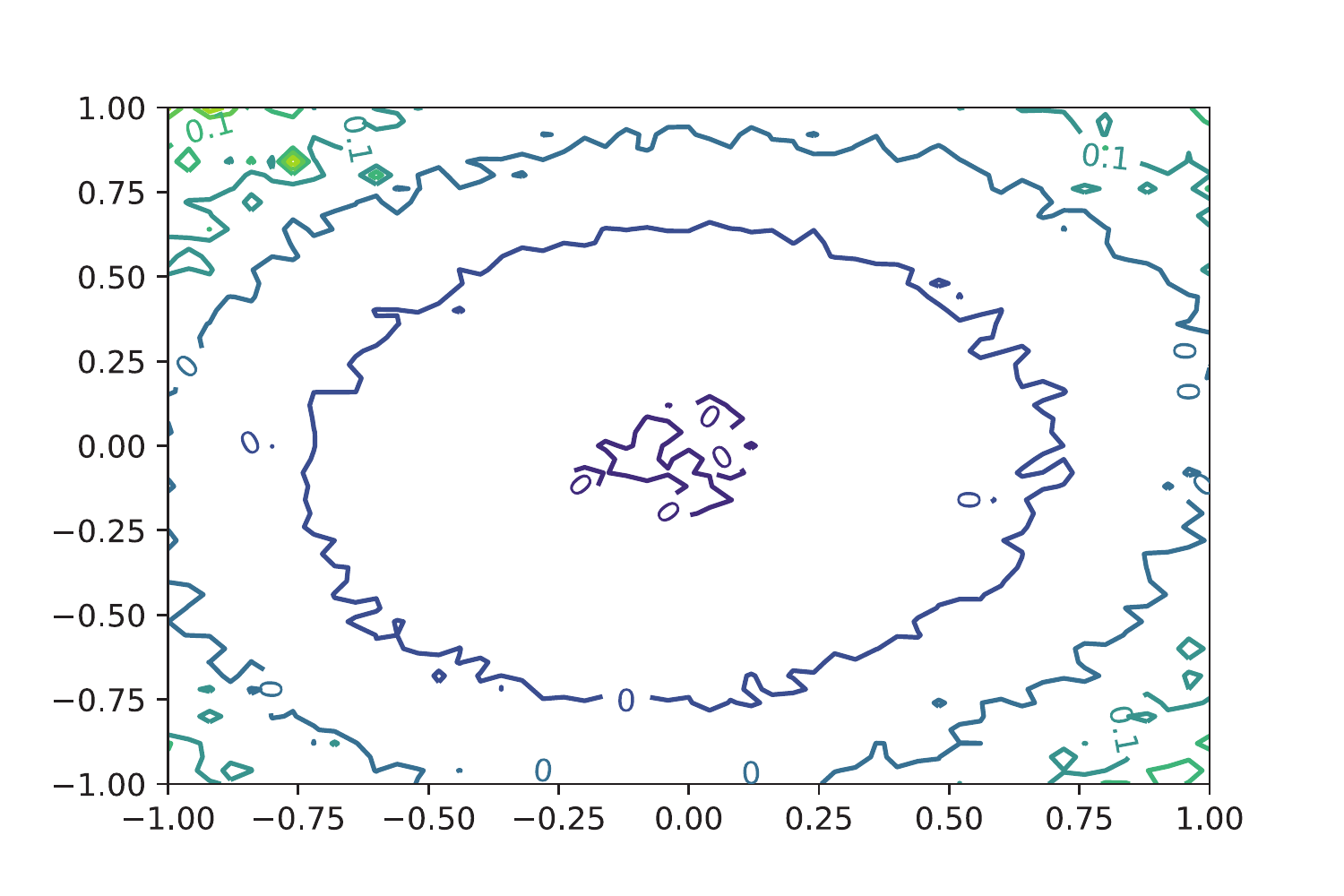}
\end{minipage}%
}%
\hspace{-2.1cm}
\subfigure[Fed-SHVR with $\alpha_1=0.75, \alpha_2=
0.1$]{
\begin{minipage}[t]{0.5\linewidth}
\centering
\includegraphics[width=2.7in]{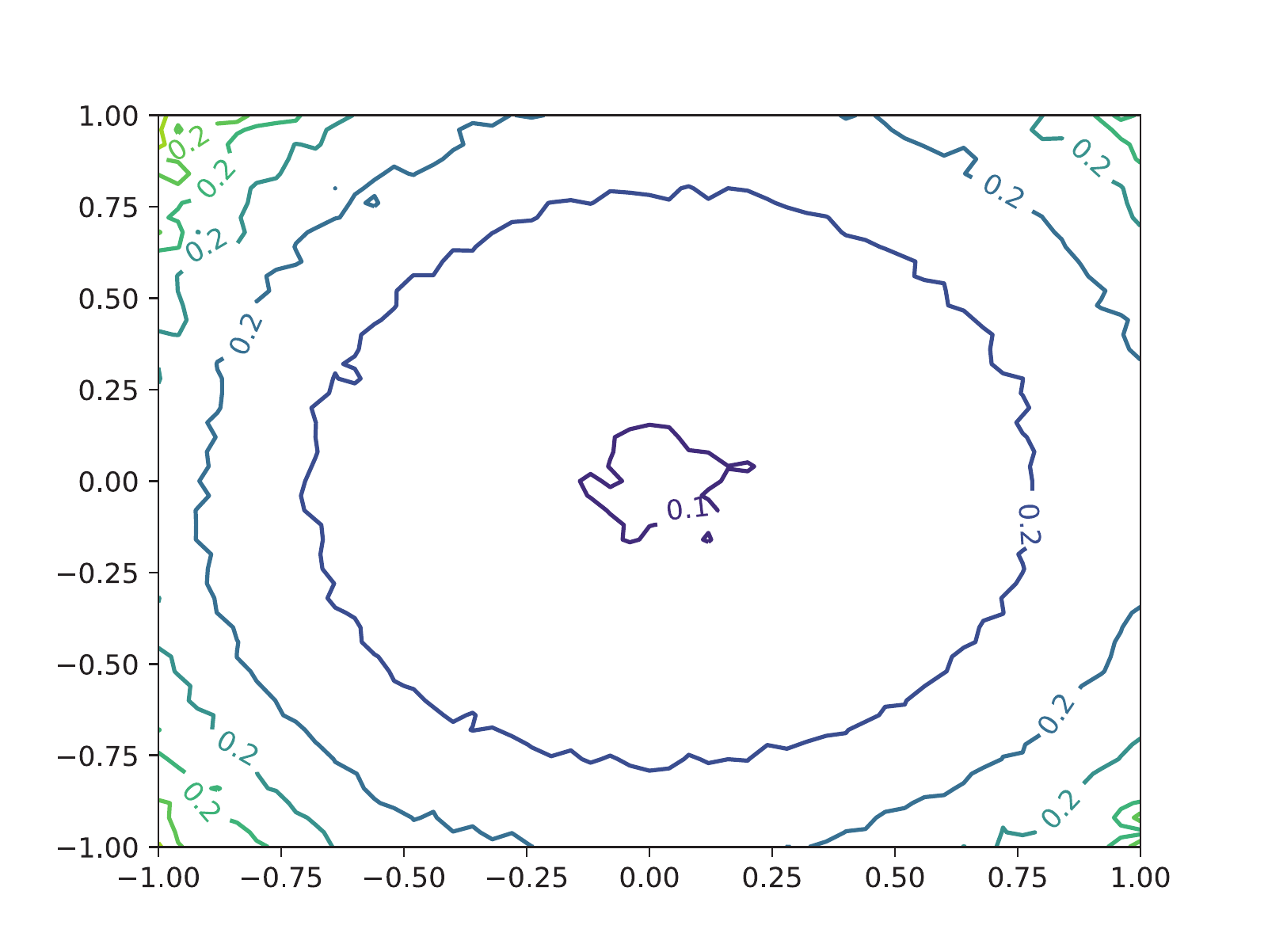}
\end{minipage}%
}%
\centering
\caption{The 2D visualization of solutions obtained using Fed-SHVR with different weight coefficients $\alpha_2$ for the model regularization term. Note that the center
of each plot corresponds to the minimizer, and the two axes parameterize two random directions in \eqref{Q1}.}
\label{landscape}
\end{figure*}

\textbf{Example 4:} We discuss the effectiveness of the two regularization terms in \eqref{r_1} and \eqref{eqn: r2}. If $\alpha_1=\alpha_2=0$, it means that  the local cost function of each client does not have any regularization, which is similar to the loss function \eqref{joint1} for the centralized pseudo label prediction problem. If $\alpha_1=0.75$ and $\alpha_2=0$, it means the local cost function only uses the regularizer for the pseudo label. If $\alpha_1=0$ and $\alpha_2=0.1$, it means the objective function only has the regularizer for the model output. If $\alpha_1=0.75$ and $\alpha_2=0.1$, it means the objective function has  two regularization terms. Table \ref{tab: accuracy_para} shows the test accuracy of the proposed Fed-SHVR method with different weight
coefficients for the two regularization terms in \eqref{eqn: FW_joint_obj}. Compared with FedAvg  that only uses labeled data, the proposed Fed-SHVR performs better whenever we use the two regularization terms, which implies that the unlabeled data can improve the performance. For the Non-IID case,   Fed-SHVR with $\alpha_1=0.75$ performs much better than that with $\alpha_1=0$ at the first communication round $T=1$. It reveals that the regularizer $r_1(\hat{\vv}_k)$ can speed up the convergence at early training iterations, especially, when the class distribution between labeled data and unlabeled data is  mismatched. Table \ref{tab: accuracy_para} also shows that the model regularizer $r_2({\theta})$ helps us to improve the final accuracy at communication round $T=100$. In other words, the confidence penalty $r_2({\theta})$  leads to smoother
output distributions, which results in better generalization. Moreover, we explore how regularizer  affects the loss landscapes on generalization \cite{li2018visualizing}. Fig. \ref{landscape} plots a function of the form
\begin{align}\label{Q1}
Q(\beta_1,\beta_2)=F(\theta^{100}+\beta_1\delta_1+\beta_2\delta_2,\hat{\vv}^{100})
\end{align}
 in the $2$D surface, where $F(\cdot,\cdot)$ is defined in \eqref{eqn: FW_joint_obj}, $\delta_1$ and $\delta_2$ are two direction vectors, $(\theta^{100},\hat{\vv}^{100})$ is model parameter and pseudo label at  communication round $T=100$. From Fig. \ref{landscape}, we observe that the loss function with $r_2({\theta})$ in \eqref{eqn: FW_joint_obj} has a smooth landscape that produces the minimizers generalize better.

\subsection{CIFAR-10}
We run experiments using CIFAR-10 benchmark, which contains 32x32 pixel RGB images belonging
to 10 output classes (``airline",``frog", ``automobile", ``bird",
``cat", ``dog",``deer",   ``horse", ``ship", ``trunk"). CIFAR-10 consists of 50000 training samples and 10000 test samples. In
our experiments, we
consider FL setting with 40 clients. In each client,
we construct the labeled data with 100 images, but these images only have 6-classes (such as, bird, cat, deer, dog, trunk, ship). Note that another client may have different 6 classes (such as, frog, horse, airline, automobile, bird, cat). Thus, the total number of labeled data is 4000 but the labeled data for the clients may have a different distribution. Let the total number of unlabeled data be 20000, which
 is randomly partitioned
across 40 clients using a Dirichlet distribution $\textmd{Dir}_{40}(0.025)$. Thus the amount of the unlabeled data in each client may be different.

Similar to \cite{guo2020safe}, we vary
the ratio of unlabeled images from 6-classes to modulate class
distribution mismatch in each client. For example, when the extent of labeled/
unlabeled class mismatch ratio is 1, all unlabeled data comes from the other 4-classes while the extent is 0.5 means half of the unlabeled data comes from classes 6-classes
and the others come from 4-classes.

All the methods in the comparison use a similar 13-layer ConvNet (CNN13) architecture with
the same initial value. CNN13 model is popularly used in SSL \cite{laine2016,tarvainen2017,miyato2018virtual}. The number of mini-batch is $B_l=32$ for labeled data, $B_u=32$ for unlabeled data. We set the learning rate $\eta=0.01$,
regularity coefficients $\alpha_1=0.75$ and $\alpha_2=0.1$. All clients perform $E=2$ local epochs but  the number of local SGD iterations $\tau_k$ is varying from 40 to 90 in our simulation.
\begin{figure*}[t]
\centering
\subfigure[mismatch ratio]{
\hspace{-0.6cm}
\begin{minipage}[t]{0.5\linewidth}
\centering
\includegraphics[width=3in]{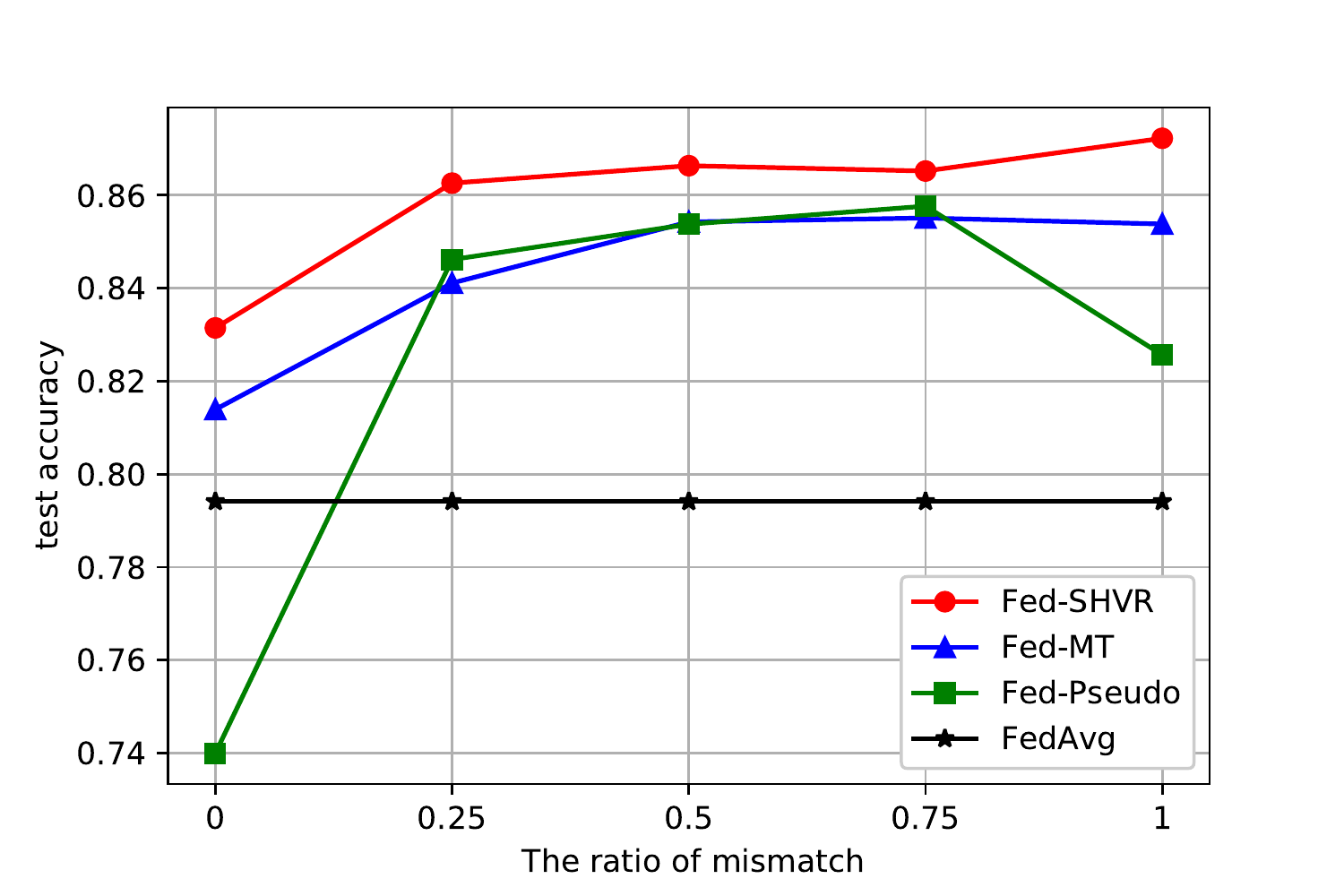}
\end{minipage}%
}%
\hspace{-2.1cm}
\subfigure[communication round]{
\hspace{-0.15cm}
\begin{minipage}[t]{0.5\linewidth}
\centering
\includegraphics[width=3in]{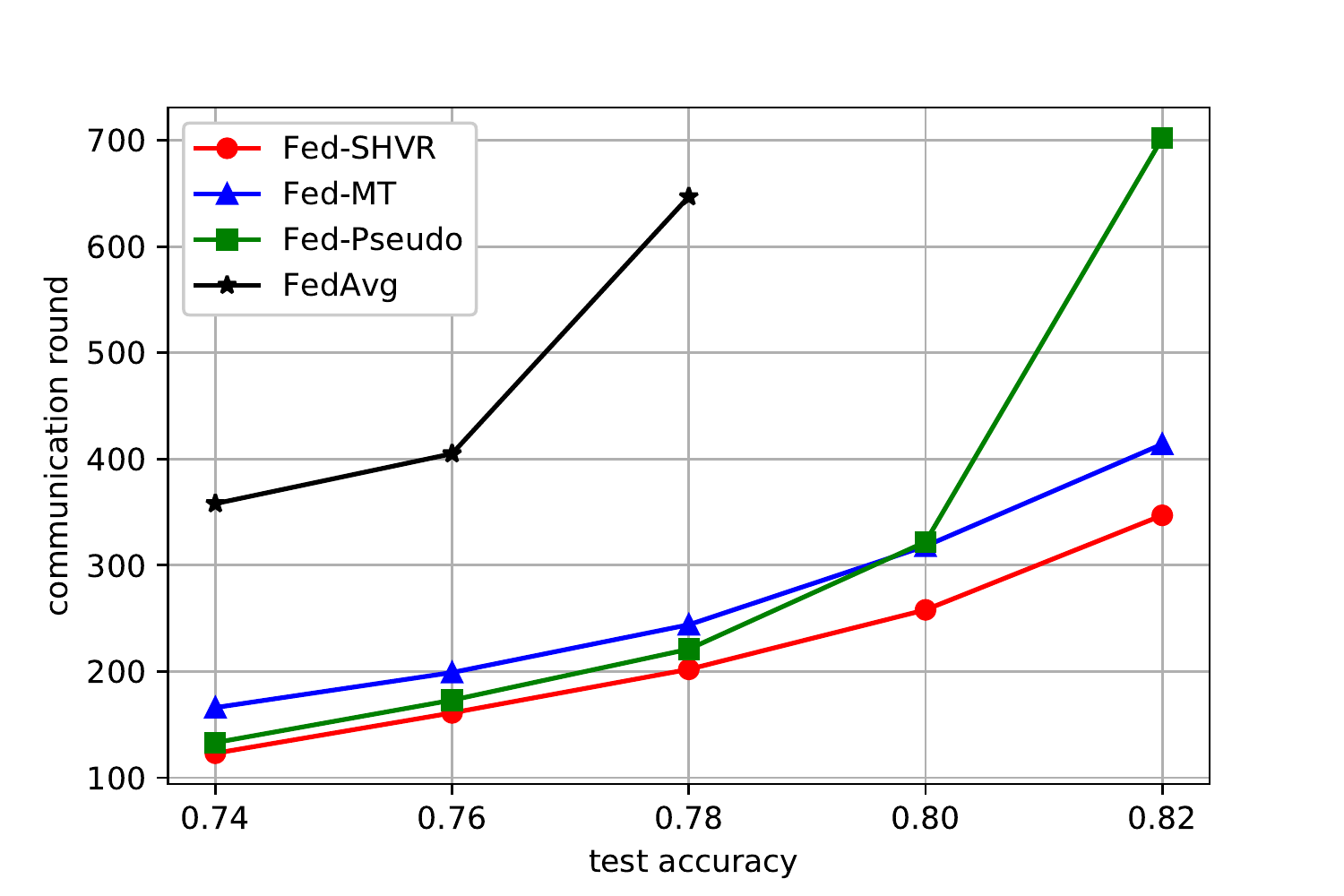}
\end{minipage}%
}%
\centering
\caption{(a): The test accuracy curves of the proposed Fed-SVHR  when varying the mismatch ratio; (b): Comparison of proposed algorithms Fed-SVHR and Fed-Pseudo, Fed-MT in terms of communication round. }
\label{fig:mismatch_commu_round_test}
\vspace{-0.1cm}
\end{figure*}

Fig. \ref{fig:mismatch_commu_round_test} (a) shows the test accuracy with 1000 communication round  of the proposed Fed-SVHR and the other federated semi-supervised methods, such as Fed-MT and Fed-Pseudo. Interestingly, it implies that the test accuracy of the proposed Fed-SVHR is raising with increasing the mismatch ratio. When the mismatch ratio is equal to 1, which means the distribution of labeled data is different from that of unlabeled data for each client. Reference \cite{guo2020safe} has shown the performance of centralized SSL method decreases significantly as
class mismatches between labeled and unlabeled data increase.
 However, the proposed semi-supervised method under federated setting does perform better than Fed-SL with the increasing of class mismatches, and one possible reason is  that Fed-SVHR and Fed-MT can use other client's information to improve the test accuracy. Comparing  Fed-SVHR with Fed-Pseudo, the proposed Fed-SVHR uses two regularization terms to eliminate the risk of class distribution mismatch. Fig. \ref{fig:mismatch_commu_round_test} (b) presents the required number of communication rounds versus the achieved test accuracy. When one wants to achieve the test accuracy with 0.76, the Fed-SSL methods  save around 200 communication rounds compared with FedAvg. In addition, the proposed Fed-SVHR uses the smallest number of communication rounds, and thus it's the most communication efficient method.

\section{Conclusion}\label{sec: conc}
We have  proposed a novel Fed-SSL algorithm for fully utilizing both labeled and unlabeled data in a heterogeneous FL setting, where the clients may have a  small amount of labeled data and a large number of unlabeled data. To overcome the challenges caused by non-i.i.d. data distribution and class distribution mismatch, we have introduced two regularization terms in \eqref{dist prob} and developed the new Fed-SSL algorithm (Algorithm \ref{alg: model_avg}), which adopts the variance reduction and normalized averaging techniques.
 We have proved that the proposed algorithm has a convergence rate of $\mathcal{O}(1/\sqrt{T})$, where $T$ is the number of communication rounds. Numerical experiments have  shown that the proposed algorithm greatly outperforms exiting Fed-SSL baselines and exhibits robust classification performance in scenarios with non-i.i.d. data and class distribution mismatch.

\vspace{-0.0cm}
		\appendices
\section{Proof of Lemma \ref{lem: pseudo}}\label{proof_lemma}
The Lagrangian function of \eqref{update_y} is given by
\begin{align}\label{LVL}
\nonumber L(\hat{\vv}_{k},\lambda) &=  -\frac{\alpha_0}{M_k} \sum_{i=1}^{M_k}\langle \hat{v}_{k,i},\log\big(f_{\theta}(u_{k,i})\big)\rangle\\
 &+\frac{\alpha_1}{M_k}\sum_{i=1}^{M_k}  {\rm KL}(\hat{v}_{k,i},\vu) +\sum_{i=1}^{M_k}\lambda_i(e^\top\hat{v}_{k,i}-1).
\end{align}
The optimal solution should satisfy the following condition
\begin{align}
\label{eqn: KKT1}&\nabla_{\hat{v}_{k,i}} L(\hat{\vv}_{k},\lambda) = 0\\
\label{eqn: KKT2}&e^\top\hat{v}_{k,i}-1=0, \hat{v}_{k,i}\geq0.
\end{align}
Substituting \eqref{LVL} into \eqref{eqn: KKT1}, it can be rewritten as
\begin{align*}
-\frac{\alpha_0}{M_k}\log\big(f_{\theta}(u_{k,i})\big)+\frac{\alpha_1}{M_k}\nabla_{\hat{v}_{k,i}} {\rm KL}(\hat{v}_{k,i},\vu)+\lambda_ie=0.
\end{align*}
Then, we get
\begin{align*}
\nabla_{\hat{v}_{k,i}} r_1(\hat{v}_{k,i})=\Big[\log\left(\frac{[\hat{v}_{k,i}]_1}{1/C}\right)+1,\ldots,
\log\left(\frac{[\hat{v}_{k,i}]_C}{1/C}\right)+1\Big]^T,
\end{align*}
where $[\hat{v}_{k,i}]_j, j=1,\ldots,C$, denote the value of class $j$. Combing the above two equations, for any $j$, we obtain
 \begin{align*}
 -\alpha_0\log\big([f_{\theta}(u_{k,i})]_j\big)+\alpha_1\left(\log\Big(\frac{[\hat{v}_{k,i}]_j}
 {1/C}\Big)+1\right)
 +M_k\lambda_i=0.
\end{align*}
Thus,
\begin{align}
\label{eqn: closed_form1}[\hat{v}_{k,i}]_j=\frac{1}{C}[f_{\theta}(u_{k,i})]_j^{\frac{\alpha_0}{\alpha_1}} \exp^{\frac{-M_k\lambda_i}{\alpha_1-1}}.
\end{align}
Using condition \eqref{eqn: KKT2},  we have
\begin{align}
e^\top \hat{v}_{k,i}=\sum_{j=1}^C\frac{1}{C}[f_{\theta}(u_{k,i})]_j^{\frac{\alpha_0}{\alpha_1}} \exp^{\frac{-M_k\lambda_i}{\alpha_1-1}}=1,
\end{align}
which gives $\exp^{\frac{-M_k\lambda_i}{\alpha_1-1}}= \frac{1}{C}\sum_{j=1}^C[f_{\theta}(u_{k,i})]_j^{\frac{\alpha_0}{\alpha_1}}$. Substituting it into \eqref{eqn: closed_form1}, we obtain the final result.
\section{The Proof of Property \ref{pro_1}}\label{proof_property}
If $g_t=0$, according to \eqref{eqn: gap}, we have
\begin{align}\label{pro_eq1}
\nabla_{\theta}F(\theta^t,\hat{\vv}^{t})=0,~
\hat{\vv}_{k}^{t+1}-\hat{\vv}_{k}^{t}=0.
\end{align}
Since $v_{k}^{t+1}$ is the optimal solution of $F_k(\theta^t,\hat{\vv})$ in \eqref{update_y}. From the first order optimality condition, we get
\begin{align}\label{pro_eq2}
\langle\nabla_{\hat{\vv}_k}F_k(\theta^t,\hat{\vv}_k^{t+1}), \hat{\vv}_k-\hat{\vv}_k^{t+1}\rangle\geq 0,~ \textmd{for} ~ \forall ~\hat{\vv}_k\in \mathcal{V}_k.
\end{align}
Combing \eqref{pro_eq1} and \eqref{pro_eq2}, for $\forall ~\theta, \hat{\vv}_k\in \mathcal{V}_k$ we have
\begin{align*}
\langle\nabla_{\theta}F(\theta^t,\hat{\vv}^{t}), \theta-\theta^t \rangle+ \sum_{k=1}^K\langle\nabla_{\hat{\vv}_k}F_k(\theta^t,\hat{\vv}_k^{t}), \hat{\vv}_k-\hat{\vv}_k^{t}\rangle\geq 0.
\end{align*}
It shows that $\{\theta^t,\hat{\vv}^t\}$ is a stationary
point to \eqref{dist prob}.
\section{The Proof of Theorem \ref{theo: 1}}\label{proof_theorem}
\subsection{Some Auxiliary Lemmas}
Next, we show two important lemmas about the descent of the objective with respect to $\theta$ and $\hat \vv$, respectively.

 For the ease of presentation, let us define the following auxiliary variables
\begin{align*}
\Xi^t&:=\sum_{k=1}^K\frac{\omega_k}{\tau_k}\sum_{q=1}^{\tau_k}
\mathbb{E}\big[\|\theta^t-\theta_k^{t,q-1}\|^2\big],\\
\Pi^t&:=\sum_{k=1}^K\frac{\omega_k}{\tau_k}\sum_{q=1}^{\tau_k}
\mathbb{E}\big[\|\theta^t-\theta_k^{t-1,q-1}\|^2\big].
\end{align*}
\begin{lemma}\label{lem: Pi}
Under Assumptions \ref{ass: smooth}-\ref{ass: un_var}, we have
\begin{align}
\nonumber\Pi^{t+1}&\leq 4(1+\eta^2L^2\bar{\tau}^2)\Xi^t +4\eta^2\bar{\tau}^2\sigma^2\sum_{k=1}^K\frac{\omega_k^2}{\tau_k}\\
&~~~~+4\eta^2\bar{\tau}^2
\mathbb{E}\big[\|\nabla_{\theta} F(\theta^{t},\hat{\vv}^{t+1})\|^2\big].
\end{align}
\end{lemma}
\begin{proof}
Substituting \eqref{update_theta} into \eqref{eqn: fedSavg theta1}, we show the relation between the iterates $\{\theta^t\}$ as follows
\begin{align}
\nonumber\theta^{t+1}&= \theta^{t} - \eta\bar{\tau} \sum_{k=1}^K\frac{\omega_k}{\tau_k}\sum_{q=1}^{\tau_k}\Big(
g_k(\theta_k^{t,q-1},\hat{\vv}_{k}^{t+1})+d_k^t\Big)\\
&=\theta^{t} - \eta\bar{\tau} \sum_{k=1}^K\frac{\omega_k}{\tau_k}\sum_{q=1}^{\tau_k}
g_k(\theta_k^{t,q-1},\hat{\vv}_{k}^{t+1}),
\label{eqn: theta_theta}
\end{align}
where the last equality dues to \eqref{eqn: d_k2}.
Using \eqref{eqn: theta_theta}, we obtain
\begin{align*}
&\theta^{t+1}-\theta_k^{t,q-1}\\
&=
\theta^{t} - \eta\bar{\tau} \sum_{l=1}^K\frac{\omega_l}{\tau_l}\sum_{q=1}^{\tau_l}
g_l(\theta_l^{t,q-1},\hat{\vv}_{l}^{t+1})-\theta_k^{t,q-1}\\
&=\theta^{t} -\theta_k^{t,q-1}-\eta\bar{\tau}\sum_{k=1}^K\omega_k\nabla_{\theta} F_k(\theta^{t},\hat{\vv}_{k}^{t+1})\\
&- \eta\bar{\tau} \sum_{k=1}^K\frac{\omega_k}{\tau_k}\sum_{q=1}^{\tau_k}
\Big(g_k(\theta_k^{t,q-1},\hat{\vv}_{k}^{t+1})-\nabla_{\theta} F_k(\theta_k^{t,q-1},\hat{\vv}_{k}^{t+1})\Big)\\
&-\eta\bar{\tau} \sum_{k=1}^K\frac{\omega_k}{\tau_k}\sum_{q=1}^{\tau_k}
\Big(\nabla_{\theta} F_k(\theta_k^{t,q-1},\hat{\vv}_{k}^{t+1})-\nabla_{\theta} F_k(\theta^{t},\hat{\vv}_{k}^{t+1})\Big)
\end{align*}
Then, using Cauchy-Schwarz inequality and Assumptions \ref{ass: smooth}-\ref{ass: un_var}, we get
\begin{align*}
&\mathbb{E}[\|\theta^{t+1}-\theta_k^{t,q-1}\|^2]\\
&\leq 4\mathbb{E}[\|\theta^{t} -\theta_k^{t,q-1}\|^2]
+4\eta^2\bar{\tau}^2
\mathbb{E}[\|\nabla_{\theta} F(\theta^{t},\hat{\vv}^{t+1})\|^2]\\
&+4\eta^2L^2\bar{\tau}^2\sum_{k=1}^K\frac{\omega_k}{\tau_k}\sum_{q=1}^{\tau_k}
\mathbb{E}[\|\theta^t-\theta_k^{t,q-1}\|^2] +4\eta^2\bar{\tau}^2\sigma^2\sum_{k=1}^K\frac{\omega_k^2}{\tau_k}.
\end{align*}
According to the definition of $\Pi^{t+1}$ and $\Xi^t$, we obtain the final result.
\end{proof}
\begin{lemma}\label{lem: xx1} Under Assumptions \ref{ass: smooth}-\ref{ass: un_var}, if
 $\eta\leq \frac{1}{L\sqrt{11a_{\tau}}}$, where $a_{\tau}=(\hat{\tau}-1)(2\hat{\tau}-1)$ and $\hat{\tau}=\max\{\tau_1,\ldots\tau_K\}$, we have
\begin{align}\label{eqn: theta_final}
\nonumber\Xi^t &\leq \frac{1}{8}\Pi^{t}+\frac{11}{8}\eta^2a_{\tau}\mathbb{E}\Big[\|
\nabla_{\theta}F(\theta^{t},\hat{\vv}^{t+1})
\|^2\Big]\\
&~~~~+\frac{11}{8}\eta^2L^2a_{\tau}\sum_{k=1}^K\omega_k\|\hat{\vv}_k^{t+1}-
\hat \vv_k^{t}\|^2+\frac{33}{8}\eta^2\sigma^2a_{\tau}.
\end{align}
\end{lemma}
\begin{proof}
Summing up all $\theta_k^{t,s}$ from $0$ to $q-1$ for \eqref{update_theta}, we obtain
$$
\theta_k^{t,q-1}=\theta^{t,0}-\eta\sum_{s=1}^{q-1}(
g_k(\theta_k^{t,s-1},\hat{\vv}_k^{t+1})
+d_k^t),
$$
and $\theta^{t,0}=\theta^t$, then it implies
\begin{align}\label{E_theta}
\mathbb{E}[\|\theta^t-\theta_k^{t,q-1}\|^2]
=\eta^2
\mathbb{E}\Big[\Big\|\sum_{s=1}^{q-1}\Big(g_k(\theta_k^{t,s-1},\hat{\vv}_k^{t+1})
+d_k^t\Big)\Big\|^2\Big]
\end{align}
Recalling  the definition of $d_k^t$ in \eqref{eqn: d_kt}, we know
\begin{align}
\nonumber
d_k^t&=
\sum_{l=1}^K\omega_l\sum_{q^\prime=1}^{\tau_l}\frac{1}{\tau_l}
g_l(\theta_l^{t-1,q^\prime-1}, \hat \vv_l^{t})-\sum_{q^\prime=1}^{\tau_k}\frac{1}{\tau_k}
g_k(\theta_k^{t-1,q^\prime-1}, \hat \vv_k^{t})
\end{align}
Thus
\begin{align}\label{ggg_k}
\nonumber&g_k(\theta_k^{t,s-1},\hat{\vv}_k^{t+1})
+d_k^t\\
\nonumber&=g_k(\theta_k^{t,s-1},\hat{\vv}_k^{t+1})-\nabla_{\theta} F_k(\theta_k^{t,s-1},\hat{\vv}_k^{t+1})\\
\nonumber&+\nabla_{\theta} F_k(\theta_k^{t,s-1},\hat{\vv}_k^{t+1})
-\nabla_{\theta}F_k(\theta^{t},\hat{\vv}_k^{t+1})\\
\nonumber&+ \sum_{l=1}^K\omega_l\sum_{q^\prime=1}^{\tau_l}\frac{1}{\tau_l}
\Big(g_l(\theta_l^{t-1,q^\prime-1}, \hat \vv_l^{t})-
\nabla_{\theta} F_l(\theta_l^{t-1,q^\prime-1}, \hat \vv_l^{t})\Big)\\
\nonumber&-\sum_{q^\prime=1}^{\tau_k}\frac{1}{\tau_k}
\Big(g_k(\theta_k^{t-1,q^\prime-1}, \hat \vv_k^{t})-
\nabla_{\theta} F_k(\theta_k^{t-1,q^\prime-1}, \hat \vv_k^{t})\Big)\\
&+\nabla_{\theta}F(\theta^{t},\hat{\vv}^{t+1})+\textmd{B}_1,
\end{align}
where
\begin{align*}
\textmd{B}_1&:=\sum_{l=1}^K\omega_l\sum_{q^\prime=1}^{\tau_l}\frac{1}{\tau_l}
\nabla_{\theta} F_l(\theta_l^{t-1,q^\prime-1}, \hat \vv_l^{t})
-\nabla_{\theta}F(\theta^{t},\hat{\vv}^{t+1})\\
&-
\sum_{q^\prime=1}^{\tau_k}\frac{1}{\tau_k}
\nabla_{\theta} F_k(\theta_k^{t-1,q^\prime-1}, \hat \vv_k^{t})+
\nabla_{\theta}F_k(\theta^{t},\hat{\vv}_k^{t+1}).
\end{align*}
Substituting \eqref{ggg_k} into \eqref{E_theta} and using Cauchy-Schwarz inequality give rise to
\begin{align}\label{Ebig}
\nonumber&\mathbb{E}\Big[\Big\|\sum_{s=1}^{q-1}\Big(g_k(\theta_k^{t,s-1},\hat{\vv}_k^{t+1})
+d_k^t\Big)\Big\|^2\Big]\\
\nonumber&\leq6(q-1)\sum_{s=1}^{q-1}\mathbb{E}\Big[\|g_k(\theta_k^{t,s-1},\hat{\vv}_k^{t+1})-\nabla_{\theta} F_k(\theta_k^{t,s-1},\hat{\vv}_k^{t+1})\|^2\\
\nonumber&+\|\nabla_{\theta} F_k(\theta_k^{t,s-1},\hat{\vv}_k^{t+1})
-\nabla_{\theta}F_k(\theta^{t},\hat{\vv}_k^{t+1})\|^2\\
\nonumber&+\Big\|\sum_{l=1}^K\omega_l\sum_{q^\prime=1}^{\tau_l}\frac{1}{\tau_l}
\Big(g_l(\theta_l^{t-1,q^\prime-1}, \hat \vv_l^{t})-
\nabla_{\theta} F_l(\theta_l^{t-1,q^\prime-1}, \hat \vv_l^{t})\Big)\Big\|^2\\
\nonumber&+\Big\|\sum_{q^\prime=1}^{\tau_k}\frac{1}{\tau_k}
\Big(g_k(\theta_k^{t-1,q^\prime-1}, \hat \vv_k^{t})-
\nabla_{\theta} F_k(\theta_k^{t-1,q^\prime-1}, \hat \vv_k^{t})\Big)\Big\|^2\\
&+\|\nabla_{\theta}F(\theta^{t},\hat{\vv}^{t+1})\|^2+\|\textmd{B}_1\|^2\Big].
\end{align}
Using the convexity of $\|\cdot\|^2$, we have
\begin{align*}
 &\Big\|\sum_{l=1}^K\omega_l\sum_{q^\prime=1}^{\tau_l}\frac{1}{\tau_l}
\Big(g_l(\theta_l^{t-1,q^\prime-1}, \hat \vv_l^{t})-
\nabla_{\theta} F_l(\theta_l^{t-1,q^\prime-1}, \hat \vv_l^{t})\Big)\Big\|^2\\
&\leq\sum_{l=1}^K\omega_l\sum_{q^\prime=1}^{\tau_l}\frac{1}{\tau_l} \Big\|
g_l(\theta_l^{t-1,q^\prime-1}, \hat \vv_l^{t})-
\nabla_{\theta} F_l(\theta_l^{t-1,q^\prime-1}, \hat \vv_l^{t})\Big\|^2.
\end{align*}
By a similar argument, one can obtain
\begin{align*}
&\Big\|\sum_{q^\prime=1}^{\tau_k}\frac{1}{\tau_k}
\Big(g_k(\theta_k^{t-1,q^\prime-1}, \hat \vv_k^{t})-
\nabla_{\theta} F_k(\theta_k^{t-1,q^\prime-1}, \hat \vv_k^{t})\Big)\Big\|^2\\
&\leq\sum_{q^\prime=1}^{\tau_k}\frac{1}{\tau_k}\Big\|
g_k(\theta_k^{t-1,q^\prime-1}, \hat \vv_k^{t})-
\nabla_{\theta} F_k(\theta_k^{t-1,q^\prime-1}, \hat \vv_k^{t})\Big\|^2.
\end{align*}
Substituting the above two inequalities into \eqref{Ebig} and using Assumptions \ref{ass: smooth}-\ref{ass: un_var},  we have the following inequality
\begin{align}
\label{eqn: xi_theta3}
\nonumber&\mathbb{E}[\|\theta^t-\theta_k^{t,q-1}\|^2]\\
\nonumber &\leq 18\eta^2(q-1)^2\sigma^2+
6\eta^2(q-1)L^2\sum_{s=1}^{q-1}\mathbb{E}\Big[\|
\theta_k^{t,s-1}-
\theta^{t}\|^2\Big]\\
&
+6\eta^2(q-1)^2\mathbb{E}\Big[\|
\nabla_{\theta}F(\theta^{t},\hat{\vv}^{t+1})
\|^2\Big] +6\eta^2(q-1)^2\mathbb{E}[\|\textmd{B}_1\|^2].
\end{align}
Combining the definition of $\Xi^t$ and \eqref{eqn: xi_theta3}, we have
\begin{align}\label{eqn: xi_theta3_sum}
\nonumber\Xi^t&\leq \eta^2\sum_{k=1}^K\frac{\omega_k}{\tau_k}\sum_{q=1}^{\tau_k}(q-1)^2
\Big[3\sigma^2+\mathbb{E}[\|\textmd{B}_1\|^2]\Big]\\
\nonumber&+
6\eta^2L^2\sum_{k=1}^K\frac{\omega_k}{\tau_k}\sum_{q=1}^{\tau_k}
(q-1)\sum_{s=1}^{\tau_k}\mathbb{E}\Big[\|
\theta_k^{t,s-1}-
\theta^{t}\|^2\Big]\\
&+6\eta^2\sum_{k=1}^K\frac{\omega_k}{\tau_k}\sum_{q=1}^{\tau_k}(q-1)^2
\mathbb{E}\Big[\|
\nabla_{\theta}F(\theta^{t},\hat{\vv}^{t+1})
\|^2\Big].
\end{align}
Note that
\begin{align*}
\sum_{q=1}^{\tau_k}(q-1)^2&=\frac{\tau_k}{6}(\tau_k-1)(2\tau_k-1), \\ \sum_{q=1}^{\tau_k}(q-1)&=\frac{\tau_k}{2}(\tau_k-1).
\end{align*}
Then
\begin{align}\label{eqn: theta1111}
\nonumber\Xi^t &\leq
\eta^2\sum_{k=1}^K\omega_k(\tau_k-1)(2\tau_k-1)\Big[3\sigma^2
+\mathbb{E}[\|\textmd{B}_1\|^2\Big]\\
\nonumber+&
3\eta^2L^2\sum_{k=1}^K\omega_k(\tau_k-1)\sum_{q=1}^{\tau_k}
\mathbb{E}\Big[\|
\theta_k^{t,q-1}-
\theta^{t}\|^2\Big]\\
 +&\eta^2\sum_{k=1}^K\omega_k(\tau_k-1)(2\tau_k-1)\mathbb{E}\Big[\|
\nabla_{\theta}F(\theta^{t},\hat{\vv}^{t+1})
\|^2\Big].
\end{align}
Let  $a_{\tau}=(\hat{\tau}-1)(2\hat{\tau}-1)$, where $\hat{\tau}=\max\{\tau_1,\ldots\tau_K\}$, we know $\tau_k-1\leq \frac{a_{\tau}}{\tau_k}$ by $\tau_k\geq 1$. Since $\sum_{k=1}^K\omega_k=1$,  we have
\begin{align}\label{eqn: theta1112}
\nonumber\Xi^t &\leq
\eta^2a_{\tau}\sum_{k=1}^K\omega_k
\mathbb{E}[\|\textmd{B}_1\|^2]+\eta^2a_{\tau}\mathbb{E}\Big[\|
\nabla_{\theta}F(\theta^{t},\hat{\vv}^{t+1})
\|^2\Big]
\\
&+
3\eta^2L^2a_{\tau}\sum_{k=1}^K\frac{\omega_k}{\tau_k}\sum_{q=1}^{\tau_k}
\mathbb{E}\Big[\|
\theta_k^{t,q-1}-
\theta^{t}\|^2\Big]+3\eta^2a_{\tau}\sigma^2.
\end{align}
According to the definition of $\textmd{B}_1$, we get
\begin{align} \label{eqn: A1}
\nonumber&\sum_{k=1}^K\omega_k\mathbb{E}[\|\textmd{B}_1\|^2] \\
\nonumber&=
\sum_{k=1}^K\omega_k
\mathbb{E}\Big[\|
\nabla_{\theta}F_k(\theta^{t},\hat{\vv}_k^{t+1})
-
\sum_{q^\prime=1}^{\tau_k}\frac{1}{\tau_k}
\nabla_{\theta} F_k(\theta_k^{t-1,q^\prime-1}, \hat \vv_k^{t})\|^2\Big]\\
\nonumber&+\mathbb{E}\Big[\|
\sum_{l=1}^K\omega_l\sum_{q^\prime=1}^{\tau_l}\frac{1}{\tau_k}
\nabla_{\theta} F_l(\theta_l^{t-1,q^\prime-1}, \hat \vv_l^{t})
-\nabla_{\theta}F(\theta^{t},\hat{\vv}^{t+1})\|^2\Big]\\
\nonumber&+2\mathbb{E}\Big\langle\sum_{l=1}^K\omega_l\sum_{q^\prime=1}^{\tau_l}\frac{1}{\tau_l}
\nabla_{\theta} F_l(\theta_l^{t-1,q^\prime-1}, \hat \vv_l^{t})
-\nabla_{\theta}F(\theta^{t},\hat{\vv}^{t+1}),\\
\nonumber&
\sum_{k=1}^K\omega_k\nabla_{\theta}F_k(\theta^{t},\hat{\vv}_k^{t+1})
-
\sum_{k=1}^K\omega_k\sum_{q^\prime=1}^{\tau_k}\frac{1}{\tau_k}
\nabla_{\theta} F_k(\theta_k^{t-1,q^\prime-1}, \hat \vv_k^{t})\Big\rangle.
\end{align}
Since $\nabla_{\theta}F(\theta^{t},\hat{\vv}^{t+1})=\sum_{k=1}^K\omega_k\nabla_{\theta}F_k(\theta^{t},\hat{\vv}_k^{t+1})
$, it implies
\begin{align}
\nonumber&\sum_{k=1}^K\omega_k\mathbb{E}[\|\textmd{B}_1\|^2] \\
\nonumber&=
\sum_{k=1}^K\omega_k
\mathbb{E}\Big[\|
\nabla_{\theta}F_k(\theta^{t},\hat{\vv}_k^{t+1})
-
\sum_{q^\prime=1}^{\tau_k}\frac{1}{\tau_k}
\nabla_{\theta} F_k(\theta_k^{t-1,q^\prime-1}, \hat \vv_k^{t})\|^2\Big]\\
\nonumber&-\mathbb{E}\Big[\|
\sum_{l=1}^K\omega_l\sum_{q^\prime=1}^{\tau_l}\frac{1}{\tau_k}
\nabla_{\theta} F_l(\theta_l^{t-1,q^\prime-1}, \hat \vv_l^{t})
-\nabla_{\theta}F(\theta^{t},\hat{\vv}^{t+1})\|^2\Big]
\end{align}
Using the convexity of $\|\cdot\|^2$ and a negative term in the right hand side of the above equality, we have
\begin{align}
\nonumber&\sum_{k=1}^K\omega_k\mathbb{E}[\|\textmd{B}_1\|^2] \\
\nonumber&\leq \sum_{k=1}^K\omega_k\sum_{q=1}^{\tau_k}\frac{1}{\tau_k}
\mathbb{E}\Big[\|
\nabla_{\theta}F_k(\theta^{t},\hat{\vv}_k^{t+1})
-
\nabla_{\theta} F_k(\theta_k^{t-1,q-1}, \hat \vv_k^{t})\|^2\Big]
\end{align}
Then, using \eqref{L_smooth} in Assumption \ref{ass: smooth}, we obtain
\begin{align}
\nonumber&\sum_{k=1}^K\omega_k\mathbb{E}[\|\textmd{B}_1\|^2]\\
&\leq L^2\sum_{k=1}^K\frac{\omega_k}{\tau_k}\sum_{q=1}^{\tau_k}
\mathbb{E}[\|\theta^{t}-\theta_k^{t-1,q-1}\|^2+\|\hat{\vv}_k^{t+1}-
\hat \vv_k^{t}\|^2].
\end{align}

Substituting \eqref{eqn: A1} into \eqref{eqn: theta1112},
Thus, we have
\begin{align}\label{eqn: theta11}
\nonumber\Xi^t &\leq 3\eta^2\sigma^2a_{\tau}+3\eta^2L^2a_{\tau}\Xi^t+\eta^2a_{\tau}\mathbb{E}\Big[\|
\nabla_{\theta}F(\theta^{t},\hat{\vv}^{t+1})
\|^2\Big]\\
&+\eta^2L^2a_{\tau}\Pi^{t}+\eta^2L^2a_{\tau}\sum_{k=1}^K\omega_k\|\hat{\vv}_k^{t+1}-
\hat \vv_k^{t}\|^2.
\end{align}
Since $\eta\leq \frac{1}{L\sqrt{11a_{\tau}}}$, it derives $1-3\eta^2L^2a_{\tau}\geq \frac{8}{11}$ and $\frac{\eta^2L^2a_{\tau}}{1-3\eta^2L^2a_{\tau}}\leq \frac{1}{8}$.
After rearranging, it gives the final result.
\end{proof}
\begin{lemma}\label{lem: decsent}
	Under Assumptions \ref{ass: reg}-\ref{ass: un_var}, if  $\eta\leq \{\frac{4\mu\alpha_1}{\bar{\tau}+a_{\tau}L^2},\frac{1}{L\sqrt{11 a_{\tau}}},\frac{1}{2L\bar{\tau}}\}$, we have
\begin{align}\label{eqn: decsent}
\nonumber&\mathbb{E}[F(\theta^{t+1},\hat{\vv}^{t+1})]+\frac{L}{8}\Pi^{t+1}\\
\nonumber&\leq \mathbb{E}[F(\theta^{t},\hat{\vv}^t)]+\frac{L}{8}\Pi^{t}
+\left(\frac{33}{8}a_{\tau}
+\frac{3}{2}\bar{\tau}^2\sum_{k=1}^K\frac{\omega_k^2}{\tau_k}\right)\eta^2\sigma^2L
\\
&-\frac{\eta\bar{\tau}}{8}\mathbb{E}[\|\nabla_{\theta}F(\theta^t,\hat{\vv}^{t+1})\|^2]
-\frac{\eta\bar{\tau}}{8}\sum_{k=1}^K\omega_k\|
\hat{\vv}_{k}^{t}-\hat{\vv}_{k}^{t+1}\|^2
\end{align}
\end{lemma}
\begin{proof}
From Assumption \ref{ass: smooth}, we take an expectation over samples to obtain
\begin{align}\label{eqn: theta_e}
\nonumber \mathbb{E}[F(\theta^{t+1},\hat{\vv}^{t+1})]&\leq\mathbb{E}[F(\theta^{t},\hat{\vv}^{t+1})]
+\frac{L}{2}\underbrace{\mathbb{E}[\|\theta^{t+1}-\theta^{t}\|^2]}_{\textmd{B}_2}
\\
&+\underbrace{\mathbb{E}[\langle \nabla_{\theta}F(\theta^t,\hat{\vv}^{t+1}),\theta^{t+1}-\theta^{t}\rangle]}_{\textmd{B}_3}
.
\end{align}
Next, we bound the term $\textmd{B}_2$ and the term $\textmd{B}_3$ in the right hand side of the inequality \eqref{eqn: theta_e}. According to \eqref{eqn: theta_theta} and \eqref{eqn: unba} in Assumption \ref{ass: un_var}, we have
\begin{align}\label{eqn: xi_theta}
\nonumber \textmd{B}_3&=-\eta\bar{\tau}\mathbb{E}[\langle \nabla_{\theta}F(\theta^t,\hat{\vv}^{t+1}),
 \sum_{k=1}^K\frac{\omega_k}{\tau_k}\sum_{q=1}^{\tau_k}
\nabla_{\theta} F_k(\theta_k^{t,q-1}, \hat \vv_k^{t+1} )\rangle]\\
\nonumber &=-\frac{\eta\bar{\tau} }{2}\mathbb{E}\Big[\|\nabla_{\theta}F(\theta^t,\hat{\vv}^{t+1})\|^2\Big]
+\frac{\eta\bar{\tau}}{2}\mathbb{E}\Big[\|\nabla_{\theta}F(\theta^t,\hat{\vv}^{t+1})\\
\nonumber&
- \sum_{k=1}^K\frac{\omega_k}{\tau_k}\sum_{q=1}^{\tau_k}
\nabla_{\theta} F_k(\theta_k^{t,q-1}, \hat \vv_k^{t+1} )\|^2\Big]\\
&-\frac{\eta\bar{\tau}}{2}\mathbb{E}\Big[\|\sum_{k=1}^K\frac{\omega_k}{\tau_k}\sum_{q=1}^{\tau_k}
\nabla_{\theta} F_k(\theta_k^{t,q-1}, \hat \vv_k^{t+1} )\|^2\Big]
\end{align}
where  the last equality based on the common equality
$-\langle a, b\rangle=\frac{1}{2}[\|a-b\|^2-\|a\|^2-\|b\|^2]$. Since
\begin{align*}
\nabla_{\theta}F(\theta^t,\hat{\vv}^{t+1})=
\sum_{k=1}^K\frac{\omega_k}{\tau_k}\sum_{q=1}^{\tau_k}
\nabla_{\theta} F_k(\theta^t, \hat \vv_k^{t+1} ),
\end{align*}
 we have
\begin{align}\label{eqn: xi_theta2}
\nonumber&\mathbb{E}\Big[\|\nabla_{\theta}F(\theta^t,\hat{\vv}^{t+1})
-\sum_{k=1}^K\frac{\omega_k}{\tau_k}\sum_{q=1}^{\tau_k}
\nabla_{\theta} F_k(\theta_k^{t,q-1}, \hat \vv_k^{t+1} )\|^2\Big]\\
\leq&L^2\sum_{k=1}^K\sum_{q=1}^{\tau_k}
\frac{\omega_k}{\tau_k}\mathbb{E}[\|\theta^t-\theta_k^{t,q-1}\|^2]=L^2\Xi^t,
\end{align}
where the  inequality dues to the convexity of $\|\cdot\|^2$ and  $L$-smooth  assumption.

Substituting \eqref{eqn: xi_theta2} into \eqref{eqn: xi_theta}, we get
\begin{align}\label{eqn: xi_theta4}
\nonumber \textmd{B}_3\leq&-\frac{\eta\bar{\tau} }{2}\mathbb{E}\Big[\|\nabla_{\theta}F(\theta^t,\hat{\vv}^{t+1})\|^2\Big]
+\frac{\eta\bar{\tau}}{2}L^2\Xi^t\\
&-\frac{\eta\bar{\tau}}{2}\mathbb{E}\Big[\|\sum_{k=1}^K\frac{\omega_k}{\tau_k}\sum_{q=1}^{\tau_k}
\nabla_{\theta} F_k(\theta_k^{t,q-1}, \hat \vv_k^{t+1} )\|^2\Big].
\end{align}

Let us bound the third term $\textmd{B}_2$ in the right hand side of inequality \eqref{eqn: theta_e}.
Using \eqref{eqn: theta_theta}, we have
\begin{align}\label{theta_m_theta11}
\nonumber \textmd{B}_2&=
\eta^2\bar{\tau}^2\mathbb{E}\Big[\| \sum_{k=1}^K\frac{\omega_k}{\tau_k}\sum_{q=1}^{\tau_k}
g_k(\theta_k^{t,q-1},\hat{\vv}_{k}^{t+1})\|^2\Big]\\
\nonumber&\leq 2\eta^2\bar{\tau}^2\mathbb{E}\Big[\|\sum_{k=1}^K\frac{\omega_k}{\tau_k}\sum_{q=1}^{\tau_k}
(g_k(\theta_k^{t,q-1},\hat{\vv}_{k}^{t+1})\\
\nonumber&\qquad\qquad-
\nabla_{\theta}F_k(\theta_k^{t,q-1},\hat{\vv}_{k}^{t+1}))\|^2\Big]\\
&+
2\eta^2\bar{\tau}^2\mathbb{E}\Big[\|\sum_{k=1}^K\frac{\omega_k}{\tau_k}\sum_{q=1}^{\tau_k}
\nabla_{\theta}F_k(\theta_k^{t,q-1},\hat{\vv}_{k}^{t+1})\|^2\Big]
\end{align}
where inequality dues to Cauchy-Schwarz inequality. Using \eqref{eqn: unba}-\eqref{eqn: var} in Assumption \eqref{ass: un_var}, we get
\begin{align}\label{theta_m_theta}
\nonumber\textmd{B}_2\leq & 2\eta^2\bar{\tau}^2\sigma^2\sum_{k=1}^K\frac{\omega_k^2}{\tau_k}\\
&+2\eta^2\bar{\tau}^2\mathbb{E}\Big[\|\sum_{k=1}^K\frac{\omega_k}{\tau_k}\sum_{q=1}^{\tau_k}
\nabla_{\theta}F_k(\theta_k^{t,q-1},\hat{\vv}_{k}^{t+1})\|^2\Big].
\end{align}
Using \eqref{eqn: theta_e}, \eqref{eqn: xi_theta4} and \eqref{theta_m_theta}, we derive
\begin{align}
\nonumber&\mathbb{E}[F(\theta^{t+1},\hat{\vv}^{t+1})]\\
\nonumber&\leq \mathbb{E}[F(\theta^{t},\hat{\vv}^{t+1})]
-\frac{\eta\bar{\tau} }{2}\mathbb{E}\Big[\|\nabla_{\theta}F(\theta^t,\hat{\vv}^{t+1})\|^2\Big]
\\
\nonumber&-(\frac{\eta\bar{\tau}}{2}-\eta^2L\bar{\tau}^2)\mathbb{E}\Big[\|\sum_{k=1}^K\frac{\omega_k}{\tau_k}\sum_{q=1}^{\tau_k}
\nabla_{\theta} F_k(\theta_k^{t,q-1}, \hat \vv_k^{t+1} )\|^2\Big]\\
\nonumber&+\frac{\eta\bar{\tau}}{2}L^2\Xi^t+\eta^2L\bar{\tau}^2\sigma^2\sum_{k=1}^K\frac{\omega_k^2}{\tau_k}
\end{align}
According to the condition $\eta\leq \frac{1}{2L\bar{\tau}}$ in Lemma
\ref{lem: decsent}, we obtain
\begin{align}\label{eqn: theta_v}
\nonumber\mathbb{E}[F(\theta^{t+1},\hat{\vv}^{t+1})]& \leq \mathbb{E}[F(\theta^{t},\hat{\vv}^{t+1})]
+\eta^2L\bar{\tau}^2\sigma^2\sum_{k=1}^K\frac{\omega_k^2}{\tau_k}\\
&+\frac{\eta\bar{\tau}}{2}L^2\Xi^t
-\frac{\eta\bar{\tau} }{2}\mathbb{E}\Big[\|\nabla_{\theta}F(\theta^t,\hat{\vv}^{t+1})\|^2\Big]
\end{align}
 Using Lemma \ref{lem: Pi}, we have
\begin{align}\label{eqn: theta_v1}
\nonumber&\mathbb{E}[F(\theta^{t+1},\hat{\vv}^{t+1})]+\frac{L}{8}\Pi^{t+1}\\
\nonumber&\leq \mathbb{E}[F(\theta^{t},\hat{\vv}^{t+1})]
-\Big(\frac{\eta\bar{\tau} }{2}-\frac{L\eta^2}{2}\bar{\tau}^2\Big)\mathbb{E}\Big[
\|\nabla_{\theta}F(\theta^t,\hat{\vv}^{t+1})\|^2\Big]
\\
&+\Big(\frac{L}{2}+\frac{\eta^2L^3}{2}\bar{\tau}^2
+\frac{\eta\bar{\tau}}{2}L^2\Big)\Xi^t
+\frac{3}{2}\eta^2L\bar{\tau}^2\sigma^2\sum_{k=1}^K\frac{\omega_k^2}{\tau_k}
\end{align}
Based on $\eta\leq \frac{1}{2L\bar{\tau}}$, it follows that
\begin{align*}
&\frac{\eta\bar{\tau} }{2}-\frac{L\eta^2}{2}\bar{\tau}^2\geq
\frac{\eta\bar{\tau} }{4}\\
&\frac{L}{2}+\frac{\eta^2L^3}{2}\bar{\tau}^2
+\frac{\eta\bar{\tau}}{2}L^2\leq L
\end{align*}
 Then \eqref{eqn: theta_v1} is relaxed as
\begin{align}\label{eqn: theta_v111}
\nonumber&\mathbb{E}[F(\theta^{t+1},\hat{\vv}^{t+1})]+\frac{L}{8}\Pi^{t+1}\\
\nonumber&\leq \mathbb{E}[F(\theta^{t},\hat{\vv}^{t+1})]+L\Xi^t
-\frac{\eta\bar{\tau} }{4}\mathbb{E}\Big[\|\nabla_{\theta}F(\theta^t,\hat{\vv}^{t+1})\|^2\Big]\\
&+\frac{3}{2}\eta^2L\bar{\tau}^2\sigma^2\sum_{k=1}^K\frac{\omega_k^2}{\tau_k}.
\end{align}
Substituting \eqref{eqn: theta_final} in Lemma \ref{lem: xx1} into \eqref{eqn: theta_v111}, we get
\begin{align}\label{eqn: theta_v112}
\nonumber&\mathbb{E}[F(\theta^{t+1},\hat{\vv}^{t+1})]+\frac{L}{8}\Pi^{t+1}\\
\nonumber&\leq \mathbb{E}[F(\theta^{t},\hat{\vv}^{t+1})]+\frac{L}{8}\Pi^{t}+\frac{11\eta^2L^3a_{\tau}}{8}\sum_{k=1}^K\omega_k\|\hat{\vv}_k^{t+1}-
\hat \vv_k^{t}\|^2\\
\nonumber&+\Big(\frac{11}{8}\eta^2L\bar{\tau}a_{\tau}-\frac{\eta\bar{\tau} }{4}\Big)\mathbb{E}\Big[\|
\nabla_{\theta}F(\theta^{t},\hat{\vv}^{t+1})
\|^2\Big]\\
&+\Big(\frac{33}{8}a_{\tau}
+\frac{3}{2}\bar{\tau}^2\sum_{k=1}^K\frac{\omega_k^2}{\tau_k}\Big)\eta^2\sigma^2L.
\end{align}
Since $\eta\leq \frac{1}{11La_{\tau}}$, it holds
\begin{align*}
\frac{11}{8}\eta^3L^2\bar{\tau}a_{\tau}-\frac{\eta\bar{\tau} }{4}&\leq -\frac{\eta\bar{\tau} }{8}.
 \end{align*}
  So \eqref{eqn: theta_v112} is relaxed as
\begin{align}\label{eqn: theta_v113}
\nonumber&\mathbb{E}[F(\theta^{t+1},\hat{\vv}^{t+1})]+\frac{L}{8}\Pi^{t+1}\\
\nonumber&\leq \mathbb{E}[F(\theta^{t},\hat{\vv}^{t+1})]+\frac{L}{8}\Pi^{t}
-\frac{\eta\bar{\tau} }{8}\mathbb{E}\Big[\|
\nabla_{\theta}F(\theta^{t},\hat{\vv}^{t+1})
\|^2\Big]\\
&
+\frac{\eta L^2}{8}\sum_{k=1}^K\omega_k\|\hat{\vv}_k^{t+1}-
\hat \vv_k^{t}\|^2+\Big(\frac{33}{8}a_{\tau}
+\frac{3}{2}\bar{\tau}^2\sum_{k=1}^K\frac{\omega_k^2}{\tau_k}\Big)\eta^2\sigma^2L.
\end{align}
According to \eqref{eqn: FW_joint_obj}, we have
\begin{align}\label{eqn: vv1}
\nonumber &F(\theta^{t},\hat{\vv}^{t+1})-F(\theta^{t},\hat{\vv}^t)\\
\nonumber&
=\alpha_0\sum_{k=1}^K\omega_k[L_{\rm CE}(\theta^{t}; u_k, \hat{\vv}_k^{t+1})-
L_{\rm CE}(\theta^{t}; u_k, \hat{\vv}_k^{t})]\\
&+\alpha_1\sum_{k=1}^K\omega_kr_1(\hat{\vv}_k^{t+1})-
\alpha_1\sum_{k=1}^K\omega_kr_1(\hat{\vv}_k^{t}).
\end{align}
In \eqref{update_y}, it shows that $v_{k}^{t+1}$ is the optimal solution. Then we using the first order condition, it implies
\begin{align}\label{eqn: vv2}
\langle \alpha_0\nabla_{\vv_k} L_{\rm CE}(\theta^{t}; u_{k}, \hat{\vv}_{k}^{t+1})
+\alpha_1 \nabla_{\vv_k} r_1(\hat{\vv}_k^{t+1}),\hat{\vv}_{k}^{t}
-\hat{\vv}_{k}^{t+1}\rangle\geq 0.
\end{align}
Since $L_{\rm CE}(\theta; u_k, \hat{\vv}_k)$ is a linear function with respect to $\hat{\vv}_k$ and $r_1(\hat{\vv}_k)$ is a strongly convex function about $\hat{\vv}_k$. Using Assumption \ref{ass: reg}, we have
\begin{align}\label{eqn: ss}
\nonumber &\alpha_0L_{\rm CE}(\theta^{t}; u_{k}, \hat{\vv}_{k}^{t})
+\alpha_1  r_1(\hat{\vv}_k^{t})\\
\nonumber&-\alpha_0L_{\rm CE}(\theta^{t}; u_{k}, \hat{\vv}_{k}^{t+1})
-\alpha_1  r_1(\hat{\vv}_k^{t+1})\\
\nonumber&\geq \langle \alpha_0\nabla_{\vv_k} L_{\rm CE}(\theta^{t}; u_{k}, \hat{\vv}_{k}^{t+1})
+\alpha_1 \nabla_{\vv_k} r_1(\hat{\vv}_k^{t+1}),\hat{\vv}_{k}^{t}-\hat{\vv}_{k}^{t+1}\rangle\\ &+\frac{\mu\alpha_1}{2}\|
\hat{\vv}_{k}^{t}-\hat{\vv}_{k}^{t+1}\|^2
\end{align}
Substituting \eqref{eqn: vv2} into \eqref{eqn: ss} gives rise to
\begin{align}\label{eqn: ss1}
\nonumber&\alpha_0L_{\rm CE}(\theta^{t}; u_{k}, \hat{\vv}_{k}^{t+1})
+\alpha_1  r_1(\hat{\vv}_k^{t+1})-\alpha_0L_{\rm CE}(\theta^{t}; u_{k}, \hat{\vv}_{k}^{t})
\\
&-\alpha_1  r_1(\hat{\vv}_k^{t})\leq
-\frac{\mu\alpha_1}{2}\|
\hat{\vv}_{k}^{t}-\hat{\vv}_{k}^{t+1}\|^2.
\end{align}
Combing \eqref{eqn: vv1} and \eqref{eqn: ss1}, we have
\begin{align}
\label{eqn: vv3} F(\theta^{t+1},\hat{\vv}^{t+1})-F(\theta^{t+1},\hat{\vv}^t)&\leq
-\frac{\mu\alpha_1}{2}\sum_{k=1}^K\omega_k\|
\hat{\vv}_{k}^{t}-\hat{\vv}_{k}^{t+1}\|^2.
\end{align}
Using \eqref{eqn: theta_v113}, \eqref{eqn: vv1} and \eqref{eqn: vv3}, we have
\begin{align}\label{eqn: theta_vvv}
\nonumber&\mathbb{E}[F(\theta^{t+1},\hat{\vv}^{t+1})]+\frac{L}{8}\Pi^{t+1}\\
\nonumber&\leq \mathbb{E}[F(\theta^{t},\hat{\vv}^t)]+\frac{L}{8}\Pi^{t}
-\frac{\eta\bar{\tau}}{8}\mathbb{E}[\|\nabla_{\theta}F(\theta^t,\hat{\vv}^{t+1})\|^2]\\
\nonumber&+\Big(\frac{33}{8}a_{\tau}
+\frac{3}{2}\bar{\tau}^2\sum_{k=1}^K\frac{\omega_k^2}{\tau_k}\Big)\eta^2\sigma^2L \\
&-\Big(\frac{\mu\alpha_1}{2}-\frac{\eta L^2a_{\tau}}{8}\Big)\sum_{k=1}^K\omega_k\|
\hat{\vv}_{k}^{t}-\hat{\vv}_{k}^{t+1}\|^2
\end{align}
 When $\eta\leq \frac{4\mu\alpha_1}{\bar{\tau}+a_{\tau}L^2}$, it follows $\frac{\mu\alpha_1}{2}-\frac{\eta L^2a_{\tau}}{8}\geq \frac{\eta\bar{\tau}}{8}$.
 Thus, we obtain the
final result.
\end{proof}
\subsection{The Proof of Theorem \ref{theo: 1}}
Based on \eqref{eqn: decsent} of Lemma \ref{lem: decsent}, we have
\begin{align}\label{eqn: gt1}
\nonumber g_t
&\leq \frac{8}{\eta \bar{\tau}}\Big(\mathbb{E}[F(\theta^{t},\hat{\vv}^t)]-\mathbb{E}[F(\theta^{t+1},\hat{\vv}^{t+1})]\Big)
+
\frac{L}
{\eta\bar{\tau}}(\Pi^{t}-\Pi^{t+1})\\
&+\frac{8}{ \bar{\tau}}\Big(\frac{33}{8}a_{\tau}
+\frac{3}{2}\bar{\tau}^2\sum_{k=1}^K\frac{\omega_k^2}{\tau_k}\Big)\eta\sigma^2L
\end{align}
Summing the above inequality above from $1$ to $T$, we know that
\begin{align}\label{eqn: gradient_dec}
\nonumber&\min_{t\in\{1,\ldots,T\}}g_t\leq\frac{1}{T}\sum_{t=1}^{T}g_t\\
\nonumber&\leq
\frac{8}{\eta T \bar{\tau}}(\mathbb{E}[F(\theta^{1},\hat{\vv}^1)]-\mathbb{E}[F(\theta^{T},\hat{\vv}^{T})])
+
\frac{L}
{\eta T\bar{\tau}}(\Pi^{1}-\Pi^{T})\\
&+\frac{8}{ \bar{\tau}}\Big(\frac{33}{8}a_{\tau}
+\frac{3}{2}\bar{\tau}^2\sum_{k=1}^K\frac{\omega_k^2}{\tau_k}\Big)\eta\sigma^2L
\end{align}
Since 0 is the lower bound for cross-entropy loss $F(\theta,\hat{\vv})$ and $\Pi^{T}$, in addition, according to the definition of $\Pi^t$, we have $\Pi^{1}=0$. Thus, substituting $\eta=\sqrt{\frac{K}{T\bar{\tau}}}$ into the right side of inequality \eqref{eqn: gradient_dec}, we obtain the final result.
\bibliographystyle{ieeetr}
\bibliography{wang_bibtex}

%
%
%
%
%
%
%
\end{document}